%% file: main.tex
\documentclass[nohyperref]{article}

\usepackage{microtype}
\usepackage{graphicx}
\usepackage{booktabs} 

\usepackage{amsmath}
\usepackage{amssymb}
\usepackage{mathtools}
\usepackage{amsthm}
\usepackage{thm-restate}

\usepackage{hyperref}

\usepackage[accepted]{icml2022}

\newtheorem{theorem}{Theorem}[section]

\newtheorem{corollary}[theorem]{Corollary}
\theoremstyle{definition}

\newtheorem{assumption}[theorem]{Assumption}
\newtheorem{observation}[theorem]{Observation}
\theoremstyle{remark}
\newtheorem{remark}[theorem]{Remark}

\graphicspath{{./figures/}}

\usepackage[textsize=tiny]{todonotes}

\usepackage{amsmath,amssymb}
\usepackage{mathtools, bm}
\usepackage{caption}
\usepackage{subcaption}

\usepackage{bbm}

\DeclareMathOperator*{\expectation}{\mathbb{E}}
\DeclareMathOperator*{\variance}{\mathbb{V}}

\icmltitlerunning{A Parametric Class of Approximate Gradient Updates for Policy Optimization}

\begin{document}

\twocolumn[
\icmltitle{A Parametric Class of Approximate Gradient Updates for Policy Optimization}

\icmlsetsymbol{equal}{*}

\begin{icmlauthorlist}
\icmlauthor{Ramki Gummadi}{goog}
\icmlauthor{Saurabh Kumar}{goog,stan}
\icmlauthor{Junfeng Wen}{jw}
\icmlauthor{Dale Schuurmans}{goog,alb}
\end{icmlauthorlist}

\icmlaffiliation{goog}{Google Research, Brain Team}
\icmlaffiliation{stan}{Stanford University}
\icmlaffiliation{jw}{Layer 6 AI}
\icmlaffiliation{alb}{University of Alberta}

\icmlcorrespondingauthor{Ramki Gummadi}{gsrk@google.com}

\icmlkeywords{Machine Learning, ICML}

\vskip 0.3in
]

\printAffiliationsAndNotice{}  
 
\begin{abstract}
	Approaches to policy optimization have been motivated from diverse principles, based on how the parametric model is interpreted (e.g. value versus policy representation) or how the learning objective is formulated, yet they share a common goal of maximizing expected return.  To better capture the commonalities and identify  key differences between  policy optimization methods, we develop a unified perspective that re-expresses the underlying updates in terms of a limited choice of gradient form and scaling function.  In particular, we identify a parameterized space of approximate gradient updates for policy optimization that is highly structured, yet covers both classical and recent examples, including PPO.  
	As a result, we obtain novel yet well motivated updates that generalize existing algorithms in a way that can deliver benefits both in terms of convergence speed and final result quality.  An experimental investigation demonstrates that the additional degrees of freedom provided in the parameterized family of updates can be leveraged to obtain non-trivial improvements both in synthetic domains and on popular deep RL benchmarks.
\end{abstract}

\section{Introduction}\label{sec:intro}
\input{sections/intro.tex}
\section{Background}\label{sec:background}
\input{sections/background.tex}

\section{Policy-based versus Value-based Updates}\label{sec:connect}
\input{sections/connect.tex}

\section{Update Variations: The Scale-Axis}\label{sec:scale_axis}
\input{sections/scale.tex}

\section{A Parametrized Class of Updates }\label{sec:unified}
\input{sections/unified.tex}

\section{Empirical Evaluation}\label{sec:eval}
\input{sections/experiments.tex}

\section{Related Work}\label{sec:related}

\input{sections/related.tex}

\section{Conclusion}\label{sec:conclusion}
\input{sections/conclusion.tex}

\section{Acknowledgements}\label{sec:ack}
We would like to thank Minmin Chen, Yingjie Miao, Yundi Qian and anonymous reviewers for constructive feedback.

\typeout{}
\bibliography{main}
\bibliographystyle{icml2022}

\clearpage
\newpage
\appendix
\onecolumn

\input{sections/appendix.tex}

\end{document}

%% file: sections/intro.tex
Policy optimization in reinforcement learning considers the problem of learning a parameterized policy that maximizes some notion of expected return,
either via direct interactions with the environment (on-policy) or via learning from a fixed dataset (off-policy).
A direct approach to policy optimization is based on the policy gradient (PG) theorem \citep{pgoriginal},  which shows how a parameterized policy can be locally improved from unbiased Monte Carlo estimates of the gradient of the expected return. 
Alternately, value-based methods attempt to solve a proxy task of fitting a parameterized action-value function,
which explicitly models returns in the environment. 

Even though the learning objectives appear different, they share a common goal of constructing an agent that maximizes returns in an unknown environment.  
Moreover, they  share a low level algorithmic primitive of incrementally updating model parameters  based on 
learning signals derived from random experience,
which suggests a deeper underlying connection.
In this paper, we consider learning a single parametric model while being agnostic about whether it is a policy or a value model. 
A key contribution in Section~\ref{sec:connect} is to derive tight constraints on gradient estimates that correspond to standard action value or policy gradient updates on a parametric model. 
These results highlight the shared sample dependent components of the updates while also emphasizing key differences, which we use to organize a \emph{form-axis} of update variations. 
This leads us to propose a novel form of policy gradient (PG) in Theorem \ref{thm:pgpb}, \emph{Policy Gradient With Policy Baseline}, that leverages a specific form of state-action baseline \cite{actionbaseline2, actionbaseline1, actionbaseline3}, typically used to provide a control variate for additional variance reduction beyond state only baselines \cite{baseline1, baseline2}. 

In Section \ref{sec:scale_axis}, an alternate axis of variation among  updates is identified,  the \emph{scale-axis}, 
which covers gradients for robust alternatives to the MSE loss, including Huber loss.
Notably, the only difference between updates along this axis lies in how return estimates are rescaled in the gradient update. In Section \ref{sec:mlu}, we show  the Maximum-Likelihood  gradient can also be approximated as a scale variation. 

Prior work has shown that accounting for distribution shift from the behavior policy is valuable \citep{impala, ppo} in the case of off-policy learning. 
In Section \ref{sec:off_policy}, we consider off-policy corrections as another instance of \emph{scale}-axis variation. 
A key insight is the central role of a simple scaling function whose two scalar inputs are sample-dependent learning signals: (1) the return error estimate and (2) the log probability ratio of importance weight. 
In Section \ref{sec:ppo}, we prove that the gradient update arising from the surrogate objective defined for Proximal Policy Optimization (PPO) \citep{ppo} also corresponds to a particular non-trivial instantiation of this scaling function.

Section \ref{sec:unified} then summarizes the range of updates we identify for both discrete and continuous action spaces, and lists some shared properties of the scaling function that underlie viable update rules. 
In Section \ref{subsec:parametric_scale}, we extrapolate these shared properties as constraints to guide the design of a simple class of scaling functions that includes both a numerically stable approximation to the maximum likelihood scaling function as well as the default  scaling function used in policy gradient and Q-learning.

Section \ref{sec:eval} conducts an experimental evaluation of novel update rules in the expanded family,
in settings ranging from a synthetic 2D bandit benchmark (Section \ref{sec:2db}), a tabular environment (Section \ref{sec:tabular}) and the MuJoCo  continuous control  benchmark versus the PPO baseline (Section \ref{sec:cce}). These results demonstrate that the form-axis and scale-axis simplify the identification of new updates that are able to
improve policy optimization algorithms in a variety of situations.

%% file: sections/background.tex

Consider the standard infinite horizon discounted Markov Decision Process model with
states indexed by $s \in \mathcal{S}$, actions by $a \in \mathcal{A}$ and a reward $r(s, a)$. 
A policy is a conditional distribution over the action space $\mathcal{A}$ for a given state, $\pi(.|s)$, satisfying the normalization constraint
$\sum_{a} \pi(a|s) = 1$ 
for all $s \in \mathcal{S}$.
Given policy $\pi$, let $Q^\pi(s, a) = \expectation_{\pi} \left[ \sum_{t=0}^\infty \gamma^t r_t | s_0=s, a_0=a \right]$ denote
the expected return starting from $(s, a)$ and then following the policy $\pi$.
The goal of policy optimization, loosely speaking, is to find a policy $\pi$ that achieves high expected returns. 
In the general case of a parametric policy, however, different policies can be optimal for different starting states. 
To resolve this ambiguity, a fixed but arbitrary initial state distribution $\mu(s)$ over $\mathcal{S}$ is typically used to define
a concrete objective as the expected return, $J_\mu(\pi) = \expectation_{s_0 \sim \mu} \left[ \sum_{t=0}^\infty \gamma^t r_t\right]$.
An alternate way to handle $J_\mu(\pi)$ is by defining a key distribution $d^\pi_\mu(s)$, the discounted state visitation 
distribution when sampling the initial state from $\mu$: $d^\pi_\mu(s) = (1-\gamma) \sum_{t=0}^\infty \gamma^t \expectation_{s_0 \sim \mu} \left[ \Pr^\pi(s_t = s | s_0)\right]$, 
where $\Pr^\pi$ represents the (unknown) transition kernel on the MDP corresponding to a given policy $\pi$. 
With this definition, it can be shown that $J_\mu(\pi) = \expectation_{(s, a) \sim D_\mu^\pi}[r(s, a)]$~\citep{puterman2014markov}, where 
\begin{equation}\label{eqn:dpimu}
D_\mu^\pi(s, a) = d^\pi_\mu(s) \pi(a|s). 
\end{equation}
When the policy is parameterized by $\theta$ as $\pi_\theta$, the classical policy gradient theorem \cite{pgoriginal} shows that 
$\nabla_\theta J_\mu(\pi_\theta)$ can also be expressed as an expectation of gradient samples over $D_\mu^\pi$: 
$\nabla_\theta J_\mu(\theta) = \expectation_{(s, a) \sim D^\pi_\mu} [Q^\pi(s, a) \nabla_\theta \log \pi_\theta(a|s)]$. 
In practice, $Q^\pi$ is generally not available to compute the gradients 
so \textbf{policy-based methods} usually replace it with an
estimate, $\hat{T}(s, a)$, based on, e.g., discounted Monte Carlo returns.
Such an estimator, however, can perform poorly in practice. 
To address this, a \emph{baseline} is typically subtracted
from the target estimate $\hat{T}(s, a)$, i.e., introducing a control variate without bias:
\begin{equation}\label{eqn:gpg}
\hat{G}_{PG}(s, a, \theta)  = (\hat{T}(s, a) - b(s))\nabla_\theta \log \pi_\theta(a|s)
\end{equation}
On the other hand, \textbf{value-based methods} like Q-learning rely on temporal differences and perform updates that do not explicitly follow the gradient of an objective function. 
Instead, they are motivated as \emph{semi-}gradients, defined by updates in a pseudo-objective that ignores certain terms in the gradient.
These are  typically considered off-policy methods and do not explicitly specify the sampling distribution for $(s, a)$ when estimating an update direction.
For instance, given a parametric model $q_\theta(s, a)$, the Q-learning update, denoted by $\hat{G}_{QL}$ can be written as:
\begin{equation}\label{eqn:gql}
	\hat{G}_{QL}(s, a, \theta) = (\hat{T}(s, a) - q_\theta(s, a)) \nabla_\theta q_\theta(s, a)
\end{equation}
where $\hat{T}(s, a) =  r(s, a) + \gamma \max_u  q_{\theta}(s', u)$ and $s'$ is a random sample for the transition kernel from $(s, a)$.
TD-learning with SARSA \citep{sutton2018reinforcement} 
on a policy $\pi$ uses a slightly different target: $\hat{T}(s, a) = r(s, a) + \gamma q_{\theta}(s', u)$, where $u \sim \pi(.|s')$, which is called the 1-step bootstrap. This estimate can have a bias when the value function is an approximation. Therefore, generalizations of the target definition based on $n$-step returns 
can also be considered, which smoothly interpolate between the unbiased Monte Carlo rollouts and the bootstrap estimates.  However, such variations on the target  are orthogonal to 
our analysis, which focuses on the structure of the gradient updates once an estimation procedure for $\hat{T}(s, a)$ is fixed.

%% file: sections/connect.tex

Consider a learnable parametric function approximator,
$q_\theta(s, a): \mathcal{S} \times \mathcal{A} \mapsto \mathbb{R}$
for which the parameters $\theta$ are updated using gradient estimates constructed from stochastic samples of experience data.
For value-based methods (e.g. DQN or TD-learning), $q_\theta(s, a)$ explicitly models some notion of an infinite horizon discounted return value starting from state $s$ and taking action $a$. 
For a policy-based method, $q_\theta(s, a)$ could instead represent the energy function that encodes the action conditional distribution for a stochastic policy $\pi_\theta$. In the case of discrete actions, $q_\theta(s, a)$ can be interpreted as the action logits or preferences~\citep{sutton2018reinforcement}.
Taking this perspective, we begin our analysis by studying the gradients for a stochastic policy $\pi_\theta(a | s)$, that denotes a properly normalized conditional likelihood for taking action $a$ at state $s$.
For discrete actions, which we assume through this section, this is equivalent to the softmax transformation that links $\pi_\theta(a|s)$ with $\{q_\theta(s, u)\}_{u \in \mathcal{A}}$:
\begin{equation}\label{eqn:softmax}
\pi_\theta(a|s) \triangleq \exp{(q_\theta(s, a))} / \sum_u \exp(q_\theta(s, u))
\end{equation}

\subsection{Policy Gradient with Policy Baseline (PGPB)}\label{sec:pgpb}
In 
(\ref{eqn:gpg}), $\hat{G}_{PG}$ can often be impractical for non-trivial problems due to its high variance, which motivates baselines, which are generally parameterized functions that can depend on both the state and the action. 
The parameters that define baselines are typically independent of the policy parameters, however.
We now consider a particular form of state-action baseline 
that uses the policy logits as the baseline under the softmax parameterization of 
(\ref{eqn:softmax}).
The primary motivation for introducing this 
is to clarify the relationship between policy-based and value-based gradients.
We begin with a result that characterizes this form of the policy gradient estimator,
whose relation to value-based gradients would be evident later in Section \ref{sec:pg_av_connect} by contrasting the terms in their respective gradients.
All proofs are provided in the Appendix.
\begin{restatable}{theorem}{pgpb}[Policy Gradient with Policy Baseline]\label{thm:pgpb}
Given $q_\theta(s, a): \mathcal{S} \times \mathcal{A} \mapsto \mathbb{R}$, let $\pi_\theta(a|s)$ be defined by 
(\ref{eqn:softmax}). Consider the  sample based gradient estimator for 
$\theta$:
\begin{align}\label{eqn:pgpb}
	& \hat{G}_{PGPB}(s, a, \theta) \triangleq  & \\
		 & \Big(\hat{T}(s, a) - q_\theta(s, a)\Big)\nabla_\theta \log \pi_\theta(a|s)	- \nabla_\theta H(\pi_\theta(.|s)) & \nonumber
\end{align} 
where $H(\pi_\theta(.|s)) = \expectation_{a \sim \pi_\theta(.|s)} \left[- \log \pi_\theta(a|s) \right]$ denotes the entropy of 
$\pi_\theta(.|s)$. Then, $\hat{G}_{PGPB}$ is an unbiased estimate of the policy gradient if $\hat{T}$ is an unbiased estimate of return:
\begin{equation}\label{eqn:exp_pgpb}
	\nabla_\theta J^\pi_\mu(\theta) = \expectation_{(s, a) \sim D^\pi_\mu} \left[ \hat{G}_{PGPB}(s, a, \theta) \right]
\end{equation}
\end{restatable}
\begin{remark}[\textbf{Entropy Term}]
Note that unlike MaxEnt-RL~\citep{haarnoja2017reinforcement}, the objective $J^\pi_\mu(\theta)$ does not include  entropy regularization in its formulation, yet an entropy term intriguingly shows up  in 
(\ref{eqn:pgpb}). 
\end{remark}
\subsection{Connecting Policy-based and Value-based Updates}\label{sec:pg_av_connect}
Consider two gradient estimators 
that can be viewed as fitting an action-value function by minimizing an objective w.r.t.\ a target estimator $\hat{T}(s, a)$,
where
the first objective is the common Mean Squared Error (MSE), while the second is the less common Mean Variance Error (MVE) \citep{flet21}.
\begin{restatable}{theorem}{msemve}\label{thm:msemve}
Consider the value-based gradients following the MSE and MVE objectives respectively defined as:
\begin{align}
	G_{MSE}(\theta) &= - \nabla_\theta \frac{1}{2} \expectation_{(s, a)} \left[ ( \hat{T}(s, a) -  q_\theta(s, a) )^2 \right] \label{eqn:jmse} \\
	G_{MVE}(\theta) &= - \nabla_\theta \frac{1}{2} \expectation_{s} \variance_{a}\Big[\hat{T}(s, a) - q_\theta(s, a) \Big], \label{eqn:jmve} 
\end{align}
where, 
$\variance_{a}$ denotes the variance of the distribution over $a$, conditioned on the state $s$.
Then, $\hat{G}_{MSE}$ and $\hat{G}_{MVE}$  below are unbiased gradient\footnote{These are technically semi-gradients because any dependence of $\hat{D}^\pi_\mu$, the assumed on-policy sampling distribution for $(s, a)$ and $\hat{T}(s, a)$ on $\theta$ are ignored, similar to most TD learning methods.} 
estimators for $G_{MSE}$ and $G_{MVE}$.
\begin{align}
	\hat{G}_{MSE}(s, a, \theta) &\triangleq  \left( \hat{T}(s, a) - q_\theta(s, a)\right) \nabla_\theta q_\theta(s, a)  \label{eqn:gmse} \\
	\hat{G}_{MVE}(s, a, \theta) &\triangleq  \Big(\hat{T}(s, a) - q_\theta(s, a)\Big) \nabla_\theta \log \pi_\theta(a|s) \label{eqn:gmve}
\end{align}
\end{restatable}
Note that value-based methods typically do not specify the replay buffer sampling distribution, in contrast
to PG, which comes with a specific on-policy sampling distribution, $D^\pi_\mu$ defined in (\ref{eqn:dpimu}). However, this distinction is not relevant when
considering the updates conditioned on a given $(s, a)$ unless off-policy corrections are considered.
We are now ready to relate the policy gradient $\hat{G}_{PGPB}$ with $\hat{G}_{MSE}$ and $\hat{G}_{MVE}$. 
\begin{corollary}\label{cor:rel}
Given any fixed target return estimator, $\hat{T}(s, a)$, 
the following relations hold between the sample-based gradient estimators at every sample $(s, a)$ and $\theta$.
\begin{align}
	\hat{G}_{MVE} - \hat{G}_{PGPB} &= \nabla_\theta H(\pi_\theta(.|s)) \label{eqn:rel1} \\
	\hat{G}_{MSE} - \hat{G}_{MVE} &\propto \sum_u \pi_\theta(u|s) \nabla_\theta q_\theta(s, u) \label{eqn:rel2} \\
	\hat{G}_{MSE} - \hat{G}_{PGPB} -& \nabla_\theta H(\pi_\theta(.|s)) \label{eqn:rel3} \\
					&\propto \sum_u \pi_\theta(u|s) \nabla_\theta q_\theta(s, u) \nonumber
\end{align}
\end{corollary}
\begin{proof}
Equation (\ref{eqn:rel1}) follows directly from combining 
(\ref{eqn:pgpb}) with 
(\ref{eqn:gmve}). 
Equation (\ref{eqn:rel2}) follows from relating $\nabla_\theta \log \pi_\theta(a|s)$ with $\nabla_\theta q_\theta(s, a)$ using Equation (\ref{eqn:softmax}).
Finally, Equation
(\ref{eqn:rel3}) follows from the previous two observations.
\end{proof}
The significance of Corollary \ref{cor:rel} is that Equations 
(\ref{eqn:rel1}), (\ref{eqn:rel2}), (\ref{eqn:rel3}) express the relationship between the three gradient forms in terms of quantities that have \emph{no dependence on the action  $a$ or the return estimate $\hat{T}(s, a)$}. The main  signal indicating quality of the chosen action $a$, quantified by $\hat{T}(s, a)$, is aggregated identically for all three updates via:
\begin{equation} \label{eqn:delta_r}
	\Delta_R \triangleq \hat{T}(s, a) - q_\theta(s, a)
\end{equation}
To summarize the three forms of updates above more explicitly in terms of the underlying parametric model $q_\theta(s, a)$, we now state a helper theorem:
\begin{restatable}{theorem}{entropy}\label{lemma:entropy}
	Let $\hat{q}_\theta$ indicate a stop-gradient operator on $q_\theta$ with respect to $\theta$. Given the softmax parameterization of $\pi_\theta$ (Equation (\ref{eqn:softmax})), 
	we have:
\begin{equation}\label{eqn:entropy_lemma}
 \nabla_\theta H(\pi_\theta(.|s) +  \nabla_\theta \expectation_{u|s \sim \pi_\theta} \left[ \hat{q}_\theta(s, u) \right] = 0
\end{equation}
\end{restatable}
Theorem \ref{lemma:entropy}, Equation
(\ref{eqn:softmax}) and Equation (\ref{eqn:delta_r}), together allow for a summary of the three updates directly in terms of $q_\theta(s, a)$ as:
\begin{align}
	\hat{G}_{PGPB} =& ~\Delta_R \Big(\nabla_\theta q_\theta(s, a) - \expectation_{u|s \sim \pi} \left[ \nabla_\theta q_\theta(s, u) \right]\Big) \nonumber \\
		       &~+ \nabla_\theta \expectation_{u|s \sim \pi_\theta} \left[ \hat{q}_\theta(s, u) \right] \label{eqn:pgpb2} \\
	\hat{G}_{MVE} =&~ \Delta_R \Big(\nabla_\theta q_\theta(s, a) - \expectation_{u|s \sim \pi} \left[ \nabla_\theta q_\theta(s, u) \right]\Big) \label{eqn:mve2} \\
	\hat{G}_{MSE}  =&~ \Delta_R \Big( \nabla_\theta q_\theta(s, a) \Big) \label{eqn:mse2}
\end{align}
We refer to the gradients listed in Equations
(\ref{eqn:pgpb2}), (\ref{eqn:mve2}), (\ref{eqn:mse2}) as update variations along the \emph{form-axis} since the key changes required to span these updates have to do with what are effectively the \emph{bias} vectors, $\expectation_{u|s \sim \pi} \left[ \nabla_\theta q_\theta(s, u)\right]$ and $\nabla_\theta \expectation_{u|s \sim \pi_\theta} \left[ \hat{q}_\theta(s, u) \right]$, that can be computed exclusively from the model $q_\theta$ without any query to the data sample, and in particular, $\Delta_R$.

%% file: sections/scale.tex
In the previous section, we discussed update variations along the \emph{form-axis}. 
In this section, an orthogonal
variation along the \emph{scale-axis} 
is proposed, 
which considers how the gradient terms are scaled based on the return estimate. 

\subsection{Error Loss Function Updates}\label{sec:error_loss}
Consider value fitting objectives arising from aggregating individual state action prediction errors via an \emph{error loss function}, $l$ which is assumed to be smooth and convex, $l: \mathbb{R} \mapsto \mathbb{R}_+$ such that $l(0) = 0$. Under these assumptions, the derivative of $l$ is given by a non-decreasing function $\lambda : \mathbb{R} \mapsto \mathbb{R}$ with $\lambda(0) = 0$. 
Given any such loss criterion $l$ on the prediction error $\Delta_R$, let $G_{E(l)}(\theta)$ denote a generalization of $G_{MSE}(\theta)$ in Equation (\ref{eqn:jmse}) that generalizes to the error loss $l(.)$ from the MSE:
\begin{equation}\label{eqn:jef}
	G_{E(l)}(\theta) = - \nabla_\theta \expectation_{(s, a)} \left[l \Big ( \hat{T}(s, a) -  q_\theta(s, a) \Big ) \right]
\end{equation}
The sample gradient estimator, $\hat{G}_\lambda$ for the scaling function $\lambda$ (or equivalently, its loss $l$) is:
\begin{align}
	\hat{G}_\lambda(s, a, \theta) &= - \nabla_\theta ~l\Big(\hat{T}(s, a) - q_\theta(s, a) \Big ) \\
				&= \lambda(\Delta_R) \nabla_\theta q_\theta(s, a)
\end{align}
Unlike the \emph{form-axis} variations which involve the previously discussed sample independent \emph{bias} vectors, we refer to this type of variation as the \emph{scale-axis}. 
The MSE update is equivalent to $l(x) = \frac{1}{2} x^2$ and $\lambda(x) = x$. 
An alternative, the Huber loss \citep{huber64}, has an $l(x)$ parameterized by a hyperparameter $\delta > 0$ as:
\begin{equation} \nonumber
	l(x) = 
	\begin{cases}
		\frac{1}{2} x^2, & \text{if } |x| < \delta \\
		\delta(|x| - \frac{1}{2} \delta), & \text{ otherwise}
	\end{cases}
\end{equation}
In this case, $\lambda(x) = \text{clip} (x, -\delta, \delta)$ is a clipped version of the identity scaling function that arises from the MSE loss. 
Clipping the TD-errors is a common heuristic in deep RL  \citep{mnih2015human}, and corresponds to a simple scale variation as shown above.

\subsection{A Maximum-Likelihood Update}\label{sec:mlu}
Next, we derive a novel scale-axis update that is inspired by a maximum likelihood objective. 
Given a target estimator, $\hat{T}$,
consider an action target distribution 
$\hat{p}_{T}(a|s) \triangleq \exp{\Big(\hat{T}(s, a) - F(\hat{T})(s)\Big)}$, where $F(T)(s) \triangleq \log{\sum_{a} \exp{T(s, a)}}$. 
The gradient of the log-likelihood for action conditionals given this target distribution is:
\begin{align}
	& G_{ML}(\theta) \triangleq \expectation_{s} \left[ \expectation_{a|s \sim \hat{p}_T} \nabla_\theta \log{\pi_\theta(a|s)} \right] \nonumber \\
	=& \expectation_{s} \left[ \expectation_{a|s \sim \hat{p}_T} \Big[ \nabla_\theta q_\theta(s, a) \Big] - \sum_u \pi_\theta(u|s) \nabla_\theta q_\theta(s, u) \right] \nonumber \\
	       = &\expectation_{s} \left[ \expectation_{a|s \sim \pi} \left[ \frac{\hat{p}_T(a|s)}{\pi_\theta(a|s)} \nabla_\theta q_\theta(s, a) \right] - \expectation_{u|s \sim \pi} \Big[ \nabla_\theta q_\theta(s, u) \Big] \right] \nonumber \\
	       = &\expectation_{(s, a) \sim \pi}  \left[ \left(\frac{\hat{p}_T(a|s)}{\pi_\theta(a|s)} - 1\right) \nabla_\theta q_\theta(s, a) \right] \label{eqn:deriv_ml} 
\end{align}
Note that the log-likelihood ratio between $\hat{p}_T$ and $\pi_\theta$ is:
\begin{align*}
	& \log \hat{p}_T(a|s) - \log \pi_\theta(a|s) \\
	& = \hat{T}(s, a) - F(\hat{T})(s) - q_\theta(s, a) + F(q_\theta)(s) \\
	& = \Delta_R + \xi(s), \text { where } \xi(s) = F(q_\theta)(s) - F(\hat{T})(s)
	,
\end{align*}
which gives an unbiased sample estimator for 
(\ref{eqn:deriv_ml}) in the form 
$(e^{\Delta_R + \xi(s)} - 1) \nabla_\theta q_\theta(s, a)$. However, computing $F(\hat{T})(s)$ 
explicitly requires defining $\hat{T}(s,a)$ for all possible actions and not just the chosen action. 
To avoid this issue, we make an approximation that $\xi(s) \approxeq 0$, which may also be justified as an 
implicit constraint on the estimated target values.
This motivates an alternate update: 
\begin{equation}\label{eqn:ml}
\hat{G}_{ML}(s, a, \theta) \triangleq \Big(e^{\Delta_R}- 1 \Big) \nabla_\theta q_\theta(s, a)
\end{equation}
This update can be recovered by $\lambda(x) = e^x - 1$, which also satisfies $\lambda(0) = 0$ and is increasing. 

In \citet{sil}, a surrogate objective derived from modifying the MSE loss for Q-learning by clipping the prediction error below at 0 is proposed, which is referred to as self imitation learning (SIL). 
We can write the gradient expression corresponding to this objective as
\begin{equation}\label{eqn:silo}
	\hat{G}_{SIL}(s, a, \theta) =  \max(\Delta_R, 0) \nabla_\theta q_\theta(s, a)
\end{equation}
which reveals that it could be considered a piece-wise linear approximation to the Maximum Likelihood gradient $\hat{G}_{ML}$ derived above since $\lambda(x) = e^x - 1 \approxeq \max (x, 0) $.

\subsection{Off-Policy Corrected Updates}\label{sec:off_policy}
In Section \ref{sec:connect}, we made an implicit assumption that the sampled actions for each update are chosen from the 
policy encoded by the current parameters, which is 
an on-policy assumption that is typically not satisfied in practice.
Accounting for a different behavior policy requires 
re-weighting the updates with off-policy importance ratio for each sample, which is an action dependent scalar signals, similar to $\Delta_R$. 
Let $\pi_b$ denote the behavior policy. Denote the log probability ratio at a given $(s,a)$ sample as:
\begin{equation}\label{eqn:delta_o}
	\Delta_O \triangleq  \log \frac{\pi_\theta(a|s)}{\pi_b(a|s)} 
\end{equation}
For practical reasons, we only consider importance correction of the action conditional 
distribution
at each state and implicitly assume that the state marginals match, an assumption which is also common in prior practical methods which make use of off-policy corrections \citep{chen19, ppo}.
Note that the two scalar quantities $\Delta_O$ and $\Delta_R$, defined in Equations 
(\ref{eqn:delta_o}) and (\ref{eqn:delta_r})  summarize all the action dependent \emph{learning signals} needed for gradient estimation on any given $(s,a)$ sample.
Accounting for importance weight correction, the updates in Equations 
(\ref{eqn:pgpb2}), (\ref{eqn:mve2}), (\ref{eqn:mse2}), (\ref{eqn:ml}) can be generalized to:
\begin{align}
	\hat{G}_{PGPB} =& ~e^{\Delta_O} \Delta_R \Big(\nabla_\theta q_\theta(s, a) - \expectation_{u|s \sim \pi} \left[ \nabla_\theta q_\theta(s, u) \right]\Big)
	\nonumber
	\\
		       &~+ \nabla_\theta \expectation_{u|s \sim \pi_\theta} \left[ \hat{q}_\theta(s, u) \right] \label{eqn:pgpb3} \\
	\hat{G}_{MVE} =&~ e^{\Delta_O}\Delta_R \Big(\nabla_\theta q_\theta(s, a) - \expectation_{u|s \sim \pi} \left[ \nabla_\theta q_\theta(s, u) \right]\Big) \label{eqn:mve3} \\
	\hat{G}_{\lambda}      =&~ e^{\Delta_O} \lambda(\Delta_R)\Big( \nabla_\theta q_\theta(s, a) \Big) \label{eqn:ml3}
\end{align}
\begin{remark}
Equation (\ref{eqn:ml3}) spans an entire class of updates by choosing $\lambda: \mathbb{R} \mapsto \mathbb{R}$ 
to be any particular non-decreasing function which vanishes at $0$. Setting $\Delta_O = 0$ recovers the on-policy special case that
ignores the importance weights.
\end{remark}
\begin{remark}
For 
(\ref{eqn:pgpb3})
the  term $\nabla_\theta \expectation_{u|s \sim \pi_\theta} \left[ \hat{q}_\theta(s, u) \right]$ does not include an off-policy correction because it can be computed explicitly without estimating the return for the chosen action sample. 
This is in contrast to the  terms involving $\Delta_R$, which are specific to the sampled action and need  importance correction whenever the actions are chosen by a divergent behavior policy.
\end{remark}
With off-policy corrections, the scale-axis variation can now include a joint dependence of the gradient estimates on $\Delta_O$ and $\Delta_R$ in a single expression, $f(\Delta_O, \Delta_R)$.
Note that the form-axis variations considered in Section \ref{sec:connect} can  be systematically generalized by replacing $\Delta_R$ with
 $f(\Delta_O, \Delta_R) = e^{\Delta_O} \Delta_R$.
So far, we have only demonstrated cases of $f(\Delta_O, \Delta_R)$ that are multiplicatively factored into terms that depend only on $\Delta_O$ and $\Delta_R$. 
However, this need not be a universal assumption, as shown next for  PPO. 

\subsection{The Proximal Policy Optimization (PPO) update}\label{sec:ppo}
PPO \citep{ppo} optimizes a surrogate objective derived from policy gradient to learn a policy parameterization directly without an explicit state-action value function.
To relate to our framework, we need to reconstruct an explicit state-action preference $q_\theta$ from $\pi_\theta$. While this can be done using the softmax assumption in the discrete action case, the more general continuous action setting is not straightforward,
as $\pi_\theta$ may only support sampling the conditional actions and estimating the conditional probability for a given state-action pair (e.g. from a closed form Gaussian).
Nevertheless, if we have an estimate for $V^\pi(s)$, the state value function, then \citet{pgqcombine} suggests that we can consider an implied estimate of the state-action values for some constant $\alpha$ of the form $\hat{q}(s, a) \approxeq \alpha (\log \pi_\theta(a|s) + H(\pi(.|s)) + \hat{V}^\pi(s)$.
Using the notation for the advantage function, $\hat{A} = \hat{T} - V^\pi(s)$, we define:
\begin{equation} \label{eqn:qhat_dest}
	\Delta_R = \hat{T} - \hat{q} = \hat{A} - \alpha \big(\log \pi(a|s) + H(\pi(.|s))\big)
\end{equation}
Given any advantage estimator (e.g., GAE($\lambda$)~\citep{schulman2016high}), this extends our generalized update rule starting from PGPB using the value notation to a policy parameterized update rule that also supports continuous actions as:
\begin{equation}\label{update:cts}
	\hat{G}_{f, \pi} \triangleq f(\Delta_O, \Delta_R) \nabla_\theta \log \pi_\theta(a|s) + \beta \nabla_\theta H(\pi_\theta(.|s))
\end{equation}
In the above equation, $\beta$ is an additional entropy  hyperparameter similar to the final term in $\hat{G}_{PGPB}$ for the special case of discrete actions from Section \ref{sec:connect}.

\begin{restatable}{theorem}{ppo}\label{thm:ppo}
The PPO surrogate objective gradient can be recovered as $\hat{G}_{f, \pi}$ in 
(\ref{update:cts}) for the gradient scaling function, $f(x, y) = e^x y ~\tau_\epsilon(x, y)$, where $\tau_\epsilon(x, y) = \mathbbm{1}_{y > 0} \mathbbm{1}_{x < \log(1 + \epsilon)} + \mathbbm{1}_{y <0} \mathbbm{1}_{x > \log(1 - \epsilon)}$ while choosing $\alpha = 0$ in 
(\ref{eqn:qhat_dest}). 
	Note that $\tau_\epsilon(x, y) \in \{0, 1\} ~ \forall x, y, \epsilon$.
\end{restatable}

%% file: sections/unified.tex

In Sections \ref{sec:connect} and \ref{sec:scale_axis}, two separate axes of variation for standard policy optimization updates were identified.
These were referred to as the \emph{form}-axis and the \emph{scale}-axis to highlight the systematic structure shared between the  updates.
For a given $f: \mathbb{R}^2 \mapsto \mathbb{R}$, the form variations include:
\begin{align}
	\mathcal{U}_{Q}(f) &\triangleq f(\Delta_O,\Delta_R)\Big( \nabla_\theta q_\theta(s, a) \Big) \\
	\mathcal{U}_{V}(f) &\triangleq f(\Delta_O,\Delta_R) \Big(\nabla_\theta q_\theta(s, a) - \expectation_{u|s \sim \pi} \left[ \nabla_\theta q_\theta(s, u) \right]\Big) \\
	\mathcal{U}_{P}(f) &\triangleq f(\Delta_O,\Delta_R) \Big(\nabla_\theta q_\theta(s, a) - \expectation_{u|s \sim \pi} \left[ \nabla_\theta q_\theta(s, u) \right]\Big) \nonumber \\
			   & \hspace{1in}+ \nabla_\theta \expectation_{u|s \sim \pi_\theta} \left[ \hat{q}_\theta(s, u) \right] \\
	\mathcal{U}_{\pi}(f) &\triangleq f(\Delta_O, \Delta_R) \nabla_\theta \log \pi_\theta(a|s) + \beta \nabla_\theta H(\pi_\theta(.|s))
\end{align}
In the above equations, the first three rules apply specifically to discrete action settings where there is an explicit state-action parameterized model. 
The final update, $\mathcal{U}_{\pi}(f)$, in contrast, applies to a direct policy parameterization $\pi_\theta$ that is unavoidable 
in continuous action settings due to the need for efficient sampling and normalization of the action distribution. 

Section \ref{sec:scale_axis} outlined a range of possibilities for the scale-axis variation of gradient updates. 
In Assumption \ref{assumption:constraints} below, we consider natural constraints that restrict the possibilities while investigating the design of new approximate scaling functions $f: \mathbb{R}^2 \mapsto \mathbb{R}$ that may not directly correspond to any prior known principle.  
\begin{assumption}\label{assumption:constraints}
Let $f: \mathbb{R}^2 \mapsto \mathbb{R}$ denote a function that captures the dependence of a gradient update on $\Delta_O, \Delta_R$. Then, we consider $f$ to be a valid scaling function for a gradient update rule only if:
\begin{enumerate}
	\item $f(\Delta_O, \Delta_R)$ is non-decreasing in $\Delta_R$, with $f(\Delta_O, 0) = 0$ for any fixed $\Delta_O$. This implies that $f(\Delta_O, \Delta_R)$ has the same sign as $\Delta_R$, i.e. $\Delta_R f(\Delta_O, \Delta_R) \ge 0$. 
	\item $|f(\Delta_O, \Delta_R)|$ is non-decreasing in $\Delta_O$ around a neighborhood of $0$  for any fixed $\Delta_R$. 
\end{enumerate}
\end{assumption}
Constraint 1 ensures that the update direction increases $q_\theta(s, a)$ when it underestimates $\hat{T}$ and decreases $q_\theta(s, a)$ when it overestimates $\hat{T}$.
Constraint 2 ensures that samples that are more off-policy have a weaker update strength in comparison to on-policy samples. This is only assumed around some neighborhood of $0$ to allow for the trust region constraints \citep{schulman2015trust} that inspire the clipping surrogate optimized by PPO. Both constraints are clearly satisfied by all of the specific updates derived in Section \ref{sec:scale_axis}. 

\subsection{A Parametric Class of Scaling Functions}
\label{subsec:parametric_scale}
We now leverage the unified view including the constraints of Assumption \ref{assumption:constraints} to  identify a numerically stable approximation, $f_{MLA}$ to the maximum likelihood scaling function, which was derived as $f_{ML}(\Delta_O, \Delta_R) = e^{\Delta_O} (e^{\Delta_R} - 1)$, but suffers from an exponential dependence on both $\Delta_O$ and $\Delta_R$.
This new approximation, derived in the appendix, is:
\begin{align}\label{eqn:ml1}
f_{MLA}(x, y) & = \\
	& \begin{cases}
		-\frac{1}{2}(1 + x)^2, \text{ if } y \le -(1 + x) \le 0  \\
		y \max \left(1 + x + \frac{y}{2}, 0\right)  \text{ otherwise}
	\end{cases} \nonumber 
\end{align}
\begin{remark}[Qualitative similarity of $f_{ML}$ to $f_{MLA}$]
When highly off-policy ($\Delta_O \ll 0$), $f_{MLA}(\Delta_O, \Delta_R)\! =\! 0$ for all values of $\Delta_R$ up to some positive threshold that depends on $\Delta_O$. This threshold increases as the sample becomes more off-policy (i.e.\  $\Delta_O \!\rightarrow\! -\infty$). For large values of $\Delta_R$, the update strength is quadratic in $\Delta_R$.  Conversely, it saturates to a negative value as $\Delta_R\! \rightarrow\! -\infty$.
\end{remark}
We now introduce a  class of  scaling functions that is derived by further approximations to $f_{MLA}$.
This class spans a range of behaviors, all satisfying the constraints of Assumption \ref{assumption:constraints}. 
\begin{align}\label{eqn:ml2}
	f_{MLA(\alpha_o, \alpha_r)}(x, y) &=  \\
	     & y \max \left( 1 + \alpha_o x + \alpha_r y , \frac{(1 + \alpha_o x)_+}{2}\right) \nonumber 
\end{align}
In 
Equation (\ref{eqn:ml2}), the constants 
$\alpha_o, \alpha_r$ denote hyper-parameters\footnote{$(x)_+$ denotes $\max(x, 0)$} that can be used to span a class of update rules corresponding to a two-dimensional search space of scaling functions.
\begin{observation}\label{obs:cases}[Recovery of Specific Scale Functions]
	With particular values of $\alpha_o, \alpha_r$, $f_{MLA(\alpha_o, \alpha_r)}$ includes\footnote{In a non-vanishing neighborhood of $(0, 0)$} the following important special cases of scaling functions that were explicitly considered during the unification.
	\begin{enumerate}
	    \item $\alpha_o = 1, \alpha_r=0.5$ approximately recovers $f_{MLA}$, and therefore $f_{ML}$. 
		\item $\alpha_o = 1, \alpha_r=0$ approximately recovers the baseline scaling function for policy gradient with importance weight corrections.
		\item $\alpha_o = 0, \alpha_r=0$ recovers the baseline identity scaling function for action value methods without importance weight corrections (e.g. Q-learning).
	\end{enumerate}
\end{observation}
\begin{proof}[Justification of Observation \ref{obs:cases}]
		For claim 1, a visualization for the projections of $f_{ML}$ and $f_{ML(1, 0.5)}$ along various slices of $\Delta_O$ and $\Delta_R$ is shown in Figures ~\ref{fig:ml1} and \ref{fig:ml2} in the Appendix.\\
		To justify claim 2, note that $f_{MLA(1, 0)}(x, y) = y \max(1 + x, \frac{(1+x)_+}{2}) = y \max(1+x, 0) \approxeq y e^x$. \\
		Claim 3 follows trivially as $f_{MLA(0, 0)}(x, y) = y$.
\end{proof}

%% file: sections/experiments.tex

The goal of the experiments is to investigate the practical utility of the the parametric approximate gradient updates proposed. Through a diverse set of evaluations, we find that structure identified in the updates is narrow enough to demonstrate the possibility of improvements to the vanilla versions both in simple settings with few confounding effects as well as a more complex deep RL benchmark setting.

\subsection{A Synthetic 2D Contextual Bandit}\label{sec:2db}
\input{sections/synthetic.tex}

\subsection{FourRoom Environment}\label{sec:tabular}
\input{sections/tabular.tex}

\subsection{Continuous Control Experiments}\label{sec:cce}
\input{sections/cce.tex}

%% file: sections/synthetic.tex
The experiments in this Section are designed to enable an exhaustive study of the dynamics of several update rules and their convergence to the optimum under the following constraints: (1) An easily visualizable 2D parameter space (2) Under-parameterization that makes value fitting non-trivial (3) A unique and tractable optimum parameter setting to analyze the gap from the optimum solution (4) Avoid confounders related to the target estimation procedures.

 \textbf{Setting:} State space is two-dimensional corresponding to a degenerate single step MDP, denoted using $x = (x_0, x_1) \in \mathbb{R}^2$. Action space 
is discrete with $a \in \{0, \ldots, 7\}$ embedded onto a unit circle with embeddings $\Psi(a) \triangleq  (\cos(2 \pi a / 8), \sin(2 \pi a / 8))$.
Reward is set as $r(x, a) = \sigma(\langle x, \Psi(a)\rangle) \in (0, 1)$,
where $\langle a, b \rangle$ denotes the dot product of $a, b$; and $\sigma$ is the sigmoid.
The dataset consists of $(x, a, r)$, with $x$ sampled from the 2D standard Gaussian and $a$ sampled uniformly in each batch.   
The model considered\footnote{The simpler alternative of $q_\theta(x, a) = \langle \left( \theta_0 x_0, \theta_1 x_1\right), \Psi(a)\rangle$ results in a policy parameterization that is scale invariant wrt $\theta$, and hence without a proper solution.} is $q_\theta(x, a) = \langle \left( \theta_0 (1 + x_0) - 1, \theta_1(1 + x_1) - 1\right), \Psi(a)\rangle$.
The true reward cannot be fit exactly under the given parameterization, but the optimal parameters for maximizing returns of the implied policy 
are $\theta^* = (1, 1)$.  
Figure \ref{fig:objective} in the Appendix visualizes the policy return objective landscape across the two-dimensional search space traversed by the various gradient update rules. 

\textbf{Results:}
\begin{figure}
		\centering
		\includegraphics[width=0.7\linewidth]{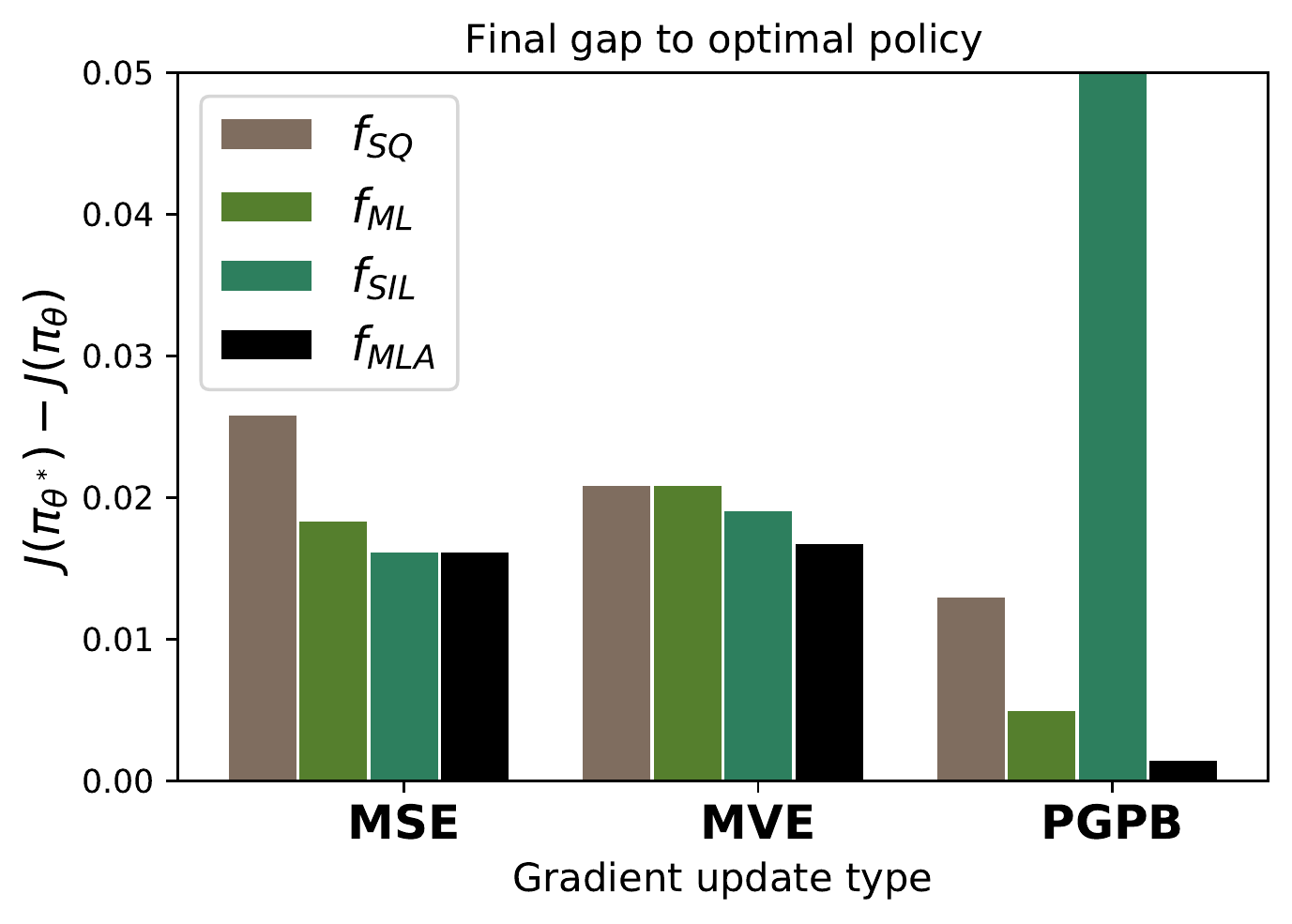}
	\caption{Comparison of update rules on a 2D synthetic bandit problem: Regret ($J(\pi_{\theta^*}) - J(\pi_\theta)$) for each of the 12 update rules considered (lower is better) at the end of 10k iterations. The updates are grouped by the three forms (MSE, MVE, PGPB), and across four scaling functions as listed in Table \ref{tab:table1} in the Appendix.}
	\label{fig:2d_synth}
\end{figure}
To compare the various update rules, trajectories of $\theta$ starting from an initialization of $(0, 0)$ are generated for each of the twelve gradient updates listed in Table \ref{tab:table1}, with more details about the setup provided in the Appendix.
Figure \ref{fig:2d_synth} shows that the best performing update is the combination of the scaling function $f_{MLA}$ with the PGPB style update.
More detailed learning curves including the dynamics of the Euclidean distance to the optimum parameters are provided in the Appendix in Figures \ref{fig:gap}, \ref{fig:err}.
Within each form of the gradient update, the maximum likelihood scaling functions performed better than the other updates in terms of both the speed of convergence as well as the final objective value. 
By contrast, the vanilla squared error objective saturates at a sub-optimal solution though it gets there much faster compared to policy gradient, while the policy gradient version converges to a good final solution, but much slower. 
These experiments show evidence that the proposed class of updates demonstrate the potential to improve both speed and final solution compared to other natural baselines.

%% file: sections/tabular.tex
To complement the other experiments, we show that the parametric scaling function can deliver benefits even while ablating the off-policy correction in the FourRoom domain~\citep{precup2000temporal}.

\textbf{Setting:}
The agent obtains a reward of 10 when they reach a pre-defined goal cell, and 0 otherwise in an environment with discount factor $0.9$ from a random initial state. 
We consider a batch/offline setting in which a dataset is collected from a uniformly random behaviour policy. 
The dataset is sufficient to cover the whole state and action space as exploration is not the focus here. 
In each gradient update, a minibatch of 64 transitions is uniformly sampled from the dataset to compute the gradients.
The policy is defined via the logits $q_\theta(s,a)$. 
Additional details are provided in the Appendix.

\textbf{Algorithms and Scaling Functions}:
We investigate the performance of the $f_{MLA(0, \alpha_r)}$ scaled update derived from two baseline updates corresponding to policy gradient and Q-learning. For policy gradient, the target $\hat{T}$ is set to be a critic learned together with the policy/actor as commonly done in the literature \citep{ours21}. 
For Q-learning, the target is bootstrapped as $r(s,a)+\gamma \max_u q_\theta(s',u)$. 
\begin{figure}
	\centering
	\begin{subfigure}{.25\textwidth}
		\centering
		\includegraphics[width=\linewidth]{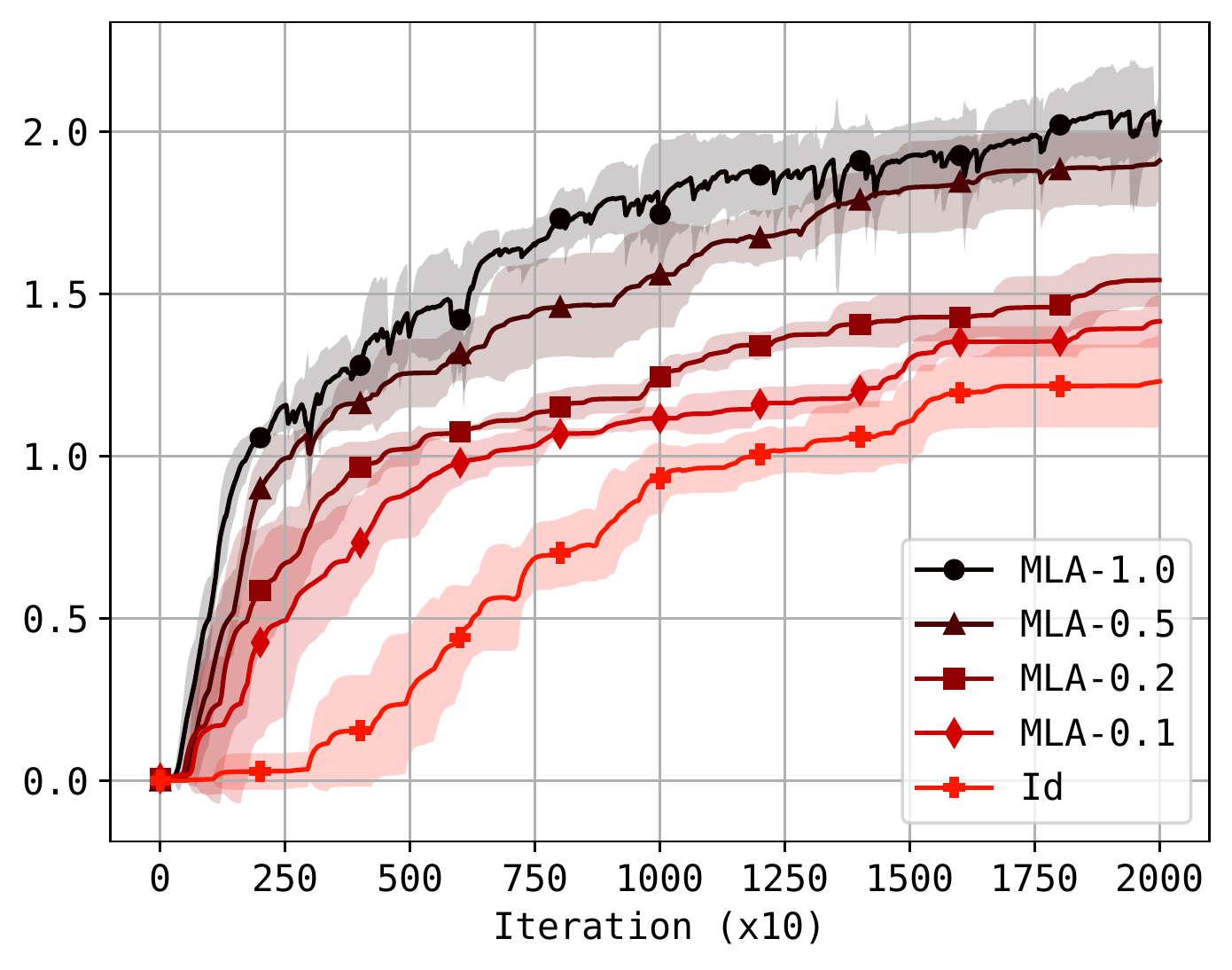}
		\caption{Policy Gradient}
		\label{subfig:fr-pg}
	\end{subfigure}%
	\begin{subfigure}{.25\textwidth}
		\centering
		\includegraphics[width=\linewidth]{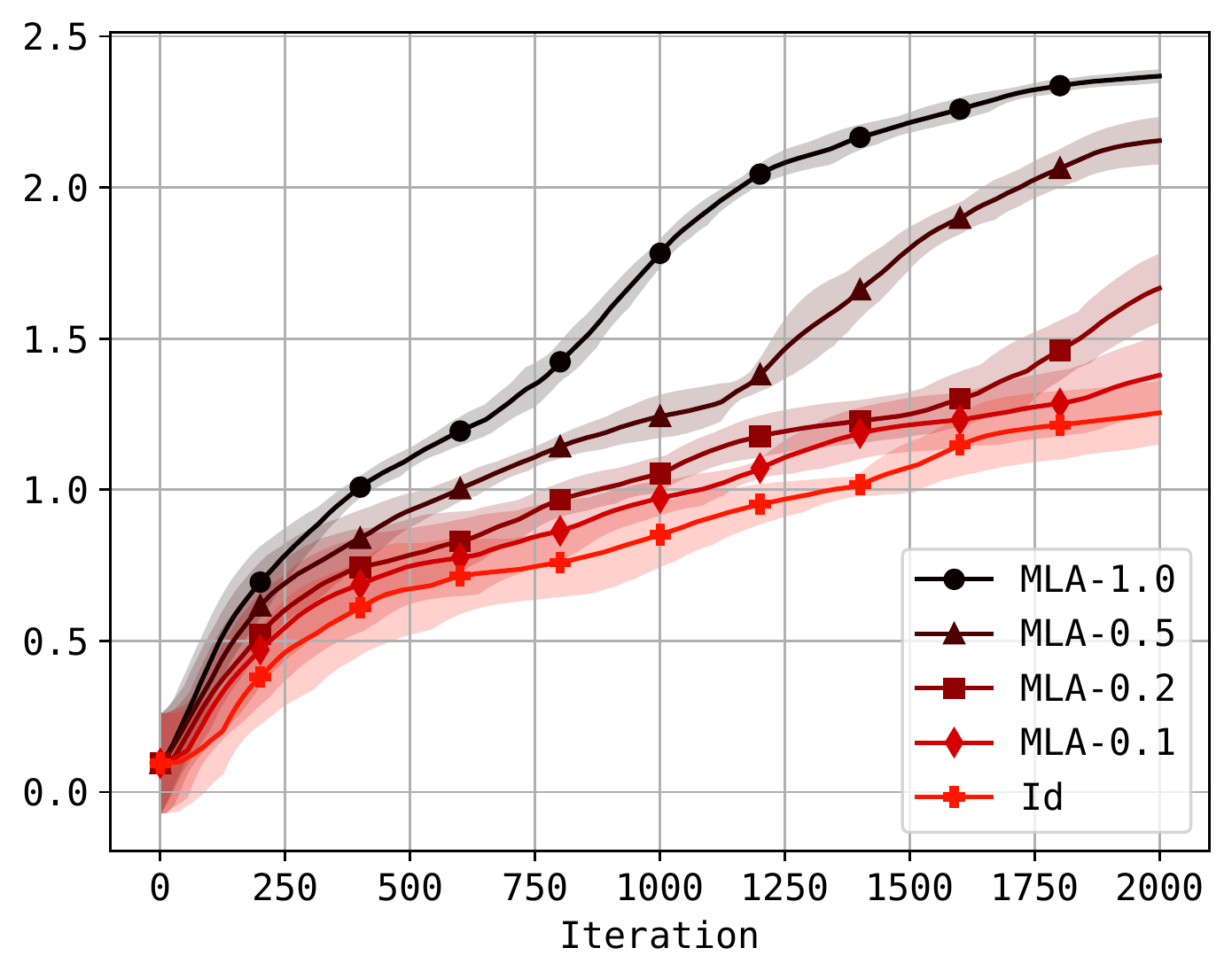}
		\caption{Q-learning}
		\label{subfig:fr-ql}
	\end{subfigure}
	\caption{FourRoom Environment: Return curves for $f_{MLA(0, \alpha)}$ across $\alpha \in \{0.1, 0.2, 0.5, 1.0\}$ against baseline for PG and Q-learning updates}
	\label{fig:fourroom-results}
\end{figure}

\textbf{Results and Discussions:}
The results are shown in Figure~\ref{fig:fourroom-results}, with mean and standard deviation over 5 runs.
The number following MLA indicates the value of $\alpha_r$ ($\alpha_o$ is set to zero).
One can see that MLA consistently improves over the identity scaling for both policy gradient and Q-learning, which are ubiquitous in the RL literature.
As discussed in Sec.~\ref{subsec:parametric_scale}, MLA corresponds to the identity scaling when $\alpha_o=\alpha_r=0$.
Figure~\ref{fig:fourroom-results} shows that the agent learns faster as $\alpha_r$ deviates from zero. Recall that $\alpha_r=0$ recovers the baseline as discussed in Sec.~\ref{subsec:parametric_scale}.
The consistent improvements of MLA for policy gradient and Q-learning with different targets indicate the significance of the scaling function for the gradient update.

%% file: sections/cce.tex
In this section, benchmarks from the MuJoCo suite are evaluated with Proximal Policy Optimization (PPO) as the baseline algorithm.
Using Theorem \ref{thm:ppo}, which shows that PPO can be equivalently implemented as a special case of the more general update rule, the experiments 
in this section are designed to validate the possibility that a systematic search over the structured scaling function space identified in Section \ref{sec:unified} can deliver non-trivial gains. 
Specifically, we consider the continuous action update rule (\ref{update:cts}) for PPO and 
replace the policy gradient scale component, $e^{\Delta_O} \Delta_R$, with the parameterized variant, $f_{MLA(\alpha_o, \alpha_r)}(\Delta_O, \Delta_R)$ from Equation (\ref{eqn:ml2}) that can span more behaviors. 
The complete expressions for the general parameterized update rule are provided in Appendix \ref{appendix:cce}. 
Under this setting, $(\alpha, \alpha_o, \alpha_r) = (0, 1, 0)$ recovers an update approximately equivalent to vanilla PPO. However, our experiments demonstrate that this is generally not the best choice as seen by the improvements in Figure \ref{fig:mujoco}. 
These results suggest that proposed update form with a flexible scaling parameterization can be a significant lever to improve state of the art algorithms.
\begin{figure}
	\centering
	\begin{subfigure}{.23\textwidth}
		\centering
		\includegraphics[width=.9\linewidth]{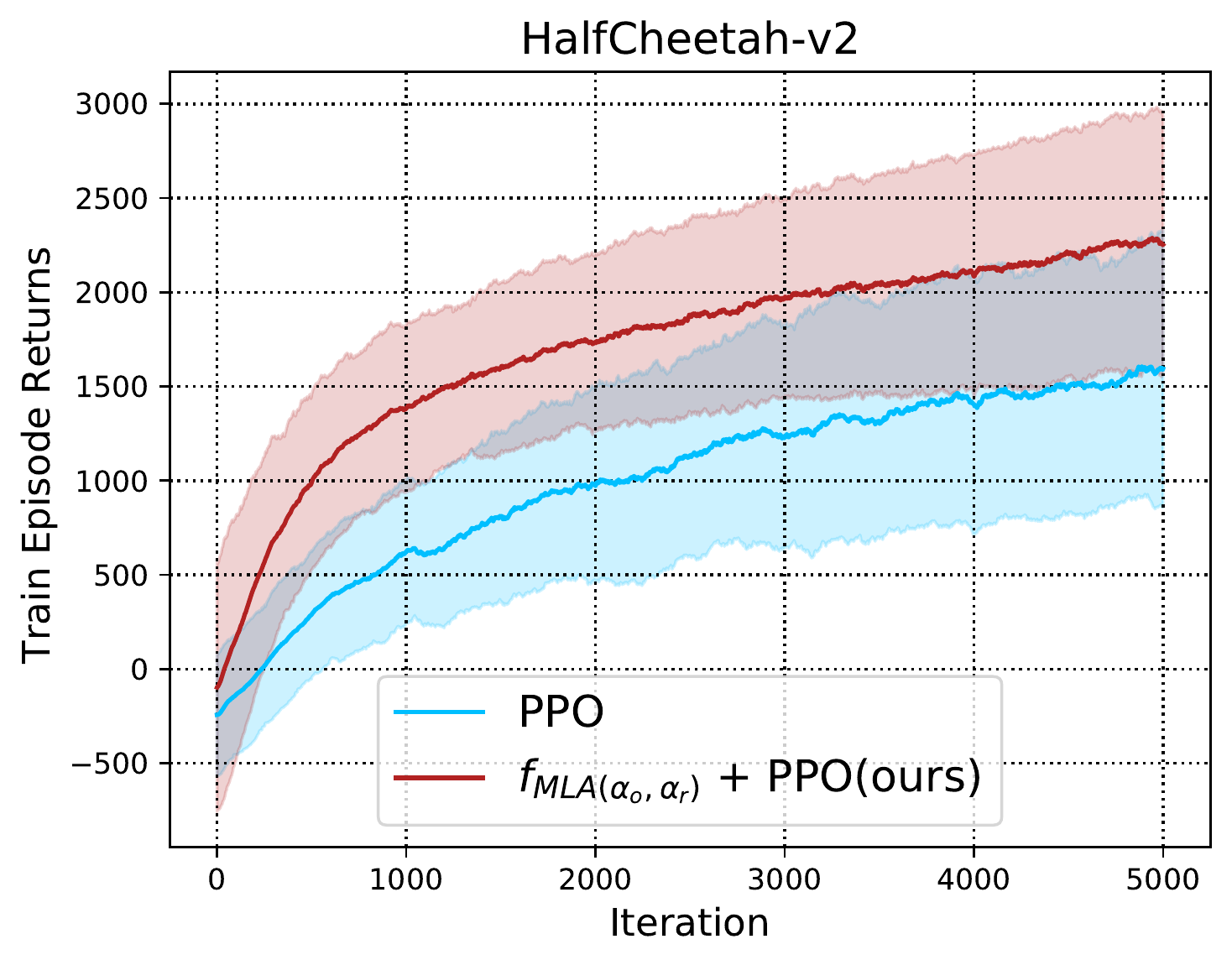}
	\end{subfigure}%
	\begin{subfigure}{.23\textwidth}
		\centering
		\includegraphics[width=.9\linewidth]{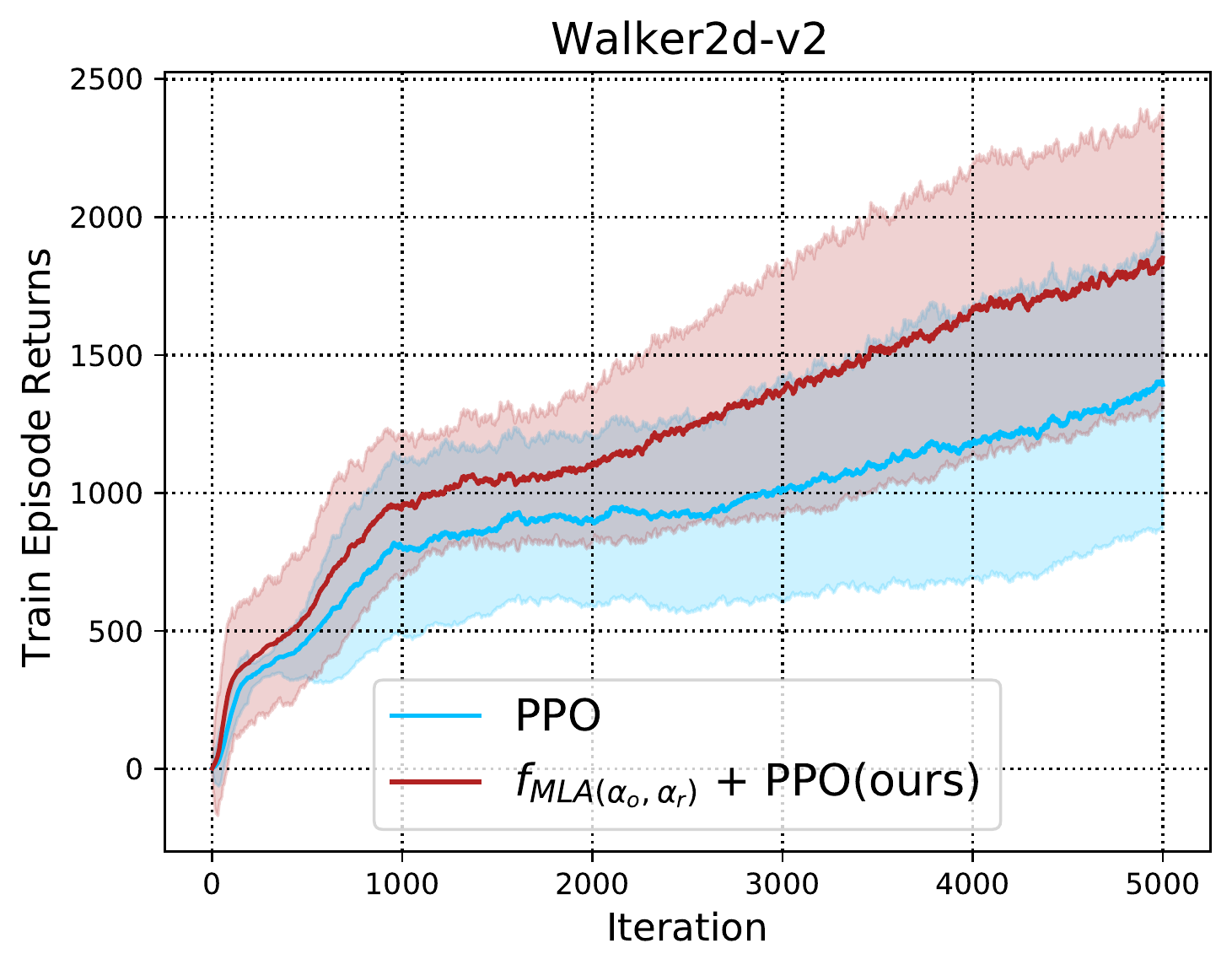}
	\end{subfigure} 
	\newline
	\begin{subfigure}{.23\textwidth}
		\centering
		\includegraphics[width=.9\linewidth]{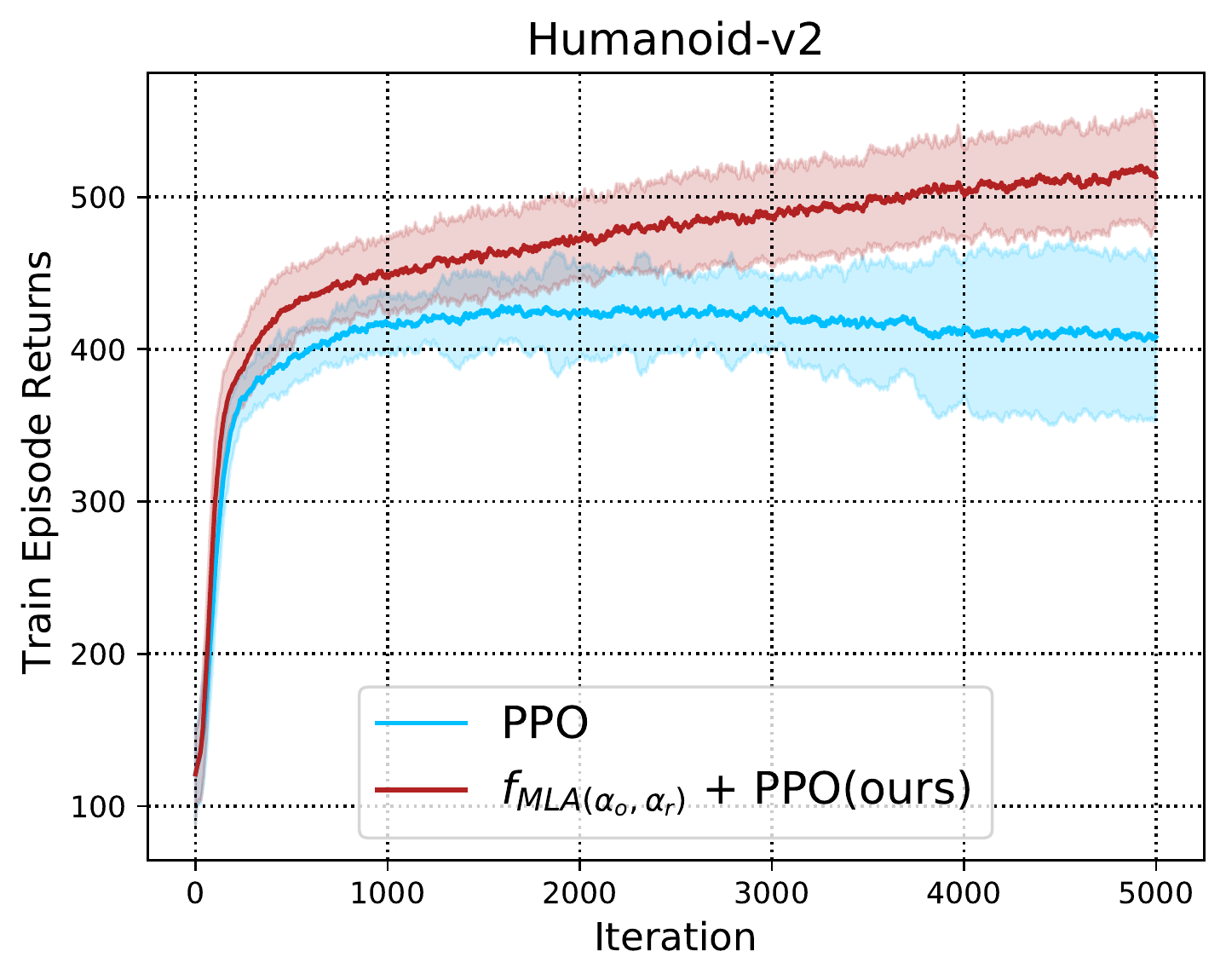}
	\end{subfigure}
	\begin{subfigure}{.23\textwidth}
		\centering
		\includegraphics[width=.9\linewidth]{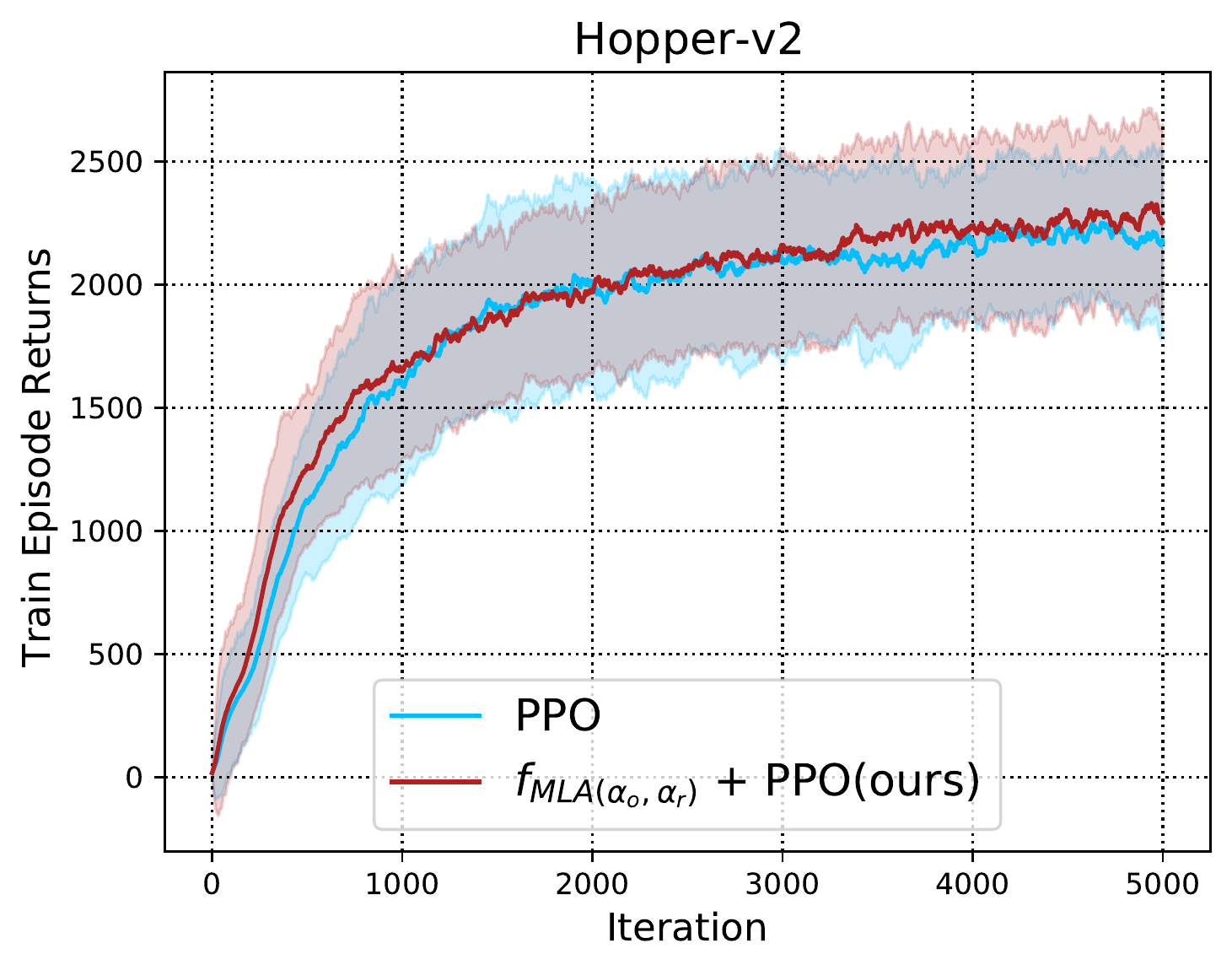}
	\end{subfigure}
	\caption{Results on MuJoCo Environments with mean and standard deviation across 100 seeds comparing $f_{MLA+PPO}$ with vanilla PPO. The x-axis corresponds to a total of 5 million environment steps.}
	\label{fig:mujoco}
\end{figure}
We used the single process agent baseline of PPO in the open source ACME framework \citep{acme} for the experiments.

%% file: sections/related.tex
\textbf{Unification of policy and value approaches:}
\citep{pgequiv, pcl, pgqcombine, GMR20} make several connections between policy and value based methods, and propose algorithms that leverage those relationships by interpreting the model as both a policy and a value parameterization. 
Traditional Actor-Critic methods~\cite{peters2008natural, konda2000actor} use value-based methods to learn a critic during policy optimization, maintaining two separate models in the process. 
Recent progress~\cite{ours21} has also led to an improved understanding of how such Actor-Critic algorithms are related to pure PG approaches. 

\textbf{Surrogate objectives and gradients:} Several prior works have also considered surrogate objectives in policy optimization
\citep{surrogate, sil, mapo, ppo, cql, serrano2021logistic}. Directly taking the gradient perspective instead of focusing on objectives has also been a focus in several prior works including \citet{pgqcombine, maei2009convergent}. 
\citet{gpi} has also proposed and studied a general class of gradient updates for policy optimization, instances of which demonstrated empirical improvements over value-based methods.
\citet{vieillard2020momentum,vieillard2020leverage} demonstrate that applying certain form of momentum/averaging to the gradient scale can be beneficial. These are orthogonal to our results, which concern novel variations for the dependence of the scale factor on the sample prediction errors.

%% file: sections/conclusion.tex


In this work, we identify structural similarities among the gradient updates for policy optimization algorithms motivated using various principles, ranging from classical policy gradients and value estimation procedures to modern examples like PPO. 
We show that policy based gradients can be directly contrasted with value based gradients when considering itself as a state-action baseline. 
In particular, they share the same dependence on the sample dependent learning signals that indicate the quality of a given action, but differ in terms of how they make use of certain other sample independent bias terms. 
We also identify a critical gradient scaling dependence function on the sample prediction errors and importance weights that can be used as a lever to span the updates corresponding to a wide variety of policy optimization procedures. 
Based on this insight, we design a simple parametric search space for valid scaling functions that can deliver reliable wins over strong baselines in a range of experimental studies from simple synthetic settings to deep RL benchmarks. 

The proposed framework has broad applicability and relevance to various policy optimization techniques. Exploring connections to adaptive step size methods \citep{dabney14}
is an interesting direction for further research. Applying the proposed parametric updates in applied settings where off policy corrections play an important role in policy optimization \citep{chen19} is another useful direction for future work.
Another interesting direction is to leverage the newly discovered structure in the approximate update space to guide the search space for Auto-RL \cite{parkerholder2022automated} techniques.

%% file: sections/appendix.tex
\begin{center}
{\huge Appendix}
\end{center}

\section{Proofs} \label{app:proofs}
\pgpb*
\begin{proof}
	Under the given assumptions on $\hat{T}(s, a)$, we have:
\begin{align}
	&\nabla_\theta J^\pi_\mu(\theta)  = \expectation_{(s, a) \sim D_\pi} \left[ \hat{T}(s, a) \nabla_\theta \log{\pi_\theta(a|s)} \right] \\
	&= \expectation_{s \sim d^\pi_\mu} \left[ \expectation_{a|s \sim \pi} \left[ (\hat{T}(s, a) - q_\theta(s, a))\nabla_\theta \log{\pi_\theta(a|s)}  \right] 
		       + \expectation_{u | s \sim \pi} \left[ q_\theta(s, u) \nabla_\theta \log{\pi_\theta(u|s)} \right] \right]\\
	&= \expectation_{s \sim d^\pi_\mu} \left[ \expectation_{a | s \sim \pi} \left[ (\hat{T}(s, a) - q_\theta(s, a))\nabla_\theta \log{\pi_\theta(a|s)}  \right] 
		       + \nabla_\theta \expectation_{u|s \sim \pi_\theta} \left[ \hat{q}_\theta(s, u) \right] \right]\\ \label{eqn:hat}
	&= \expectation_{(s, a) \sim D_\pi} \left[ \Big(\hat{T}(s, a) - q_\theta(s, a)\Big) \nabla_\theta \log{\pi_\theta(u|s)} + \nabla_\theta \expectation_{u|s \sim \pi_\theta} \left[ \hat{q}_\theta(s, u) \right] \right]  \\
	&= \expectation_{(s, a) \sim D_\pi} \left[ \Big(\hat{T}(s, a) - q_\theta(s, a)\Big) \nabla_\theta \log{\pi_\theta(u|s)} - \nabla_\theta H(\pi_\theta(.|s)  \right]  \text{ using Theorem \ref{lemma:entropy}} \\
	&= \expectation_{(s, a) \sim D_\pi} \left[ \hat{G}_{PGPB}(s, a, \theta) \right]
\end{align}
\end{proof}

\entropy*
\begin{proof}
\begin{align*}
&	\nabla_\theta \expectation_{u|s \sim \pi_\theta} \left[ \hat{q}_\theta(s, u) \right] = \sum_u q_\theta(s, u) \nabla_{\theta} \pi_\theta(u|s) \\
											     & = \sum_u \pi_\theta(u|s) q_\theta(s, u) \nabla_\theta \log \pi_\theta(u|s) \\
											     &= \sum_u \pi_\theta(u|s) (\log \pi_\theta(u|s) + F(q_\theta(s))) \nabla_\theta \log \pi_\theta(u|s) \\
											     &= \sum_u \pi_\theta(u|s) \log \pi_\theta(u|s) \nabla_\theta \log \pi_\theta(u|s)  + F(q_\theta(s))  \sum_u \pi_\theta(u|s) \nabla_\theta \log \pi_\theta(u|s)  \\
											     &= \sum_u \pi_\theta(u|s) \log \pi_\theta(u|s) \nabla_\theta \log \pi_\theta(u|s)  + F(q_\theta(s))  \sum_u \nabla_\theta \pi_\theta(u|s) \\
											     &= \sum_u \pi_\theta(u|s) \log \pi_\theta(u|s) \nabla_\theta \log \pi_\theta(u|s)  + F(q_\theta(s))  \nabla_\theta \sum_u \pi_\theta(u|s) \\
											     &= \sum_u \pi_\theta(u|s) \log \pi_\theta(u|s) \nabla_\theta \log \pi_\theta(u|s)  + F(q_\theta(s))  \nabla_\theta 1 \\
											     &= \sum_u \pi_\theta(u|s) \log \pi_\theta(u|s) \nabla_\theta \log \pi_\theta(u|s)  + \nabla_\theta 1 \\
											     &= \sum_u \pi_\theta(u|s) \log \pi_\theta(u|s) \nabla_\theta \log \pi_\theta(u|s)  + \sum_u \pi_\theta(u|s) \nabla_\theta \log \pi_\theta(u|s) \\
											     &= \sum_u \log \pi_\theta(u|s) \nabla_\theta \pi_\theta(u|s)  + \sum_u \pi_\theta(u|s) \nabla_\theta \log \pi_\theta(u|s) \\
											     &= \sum_u \nabla_\theta \left( \pi_\theta(u|s) \log \pi_\theta(u|s) \right) \\
											     & = - \nabla_\theta H(\pi_\theta(.|s) 
\end{align*}
\end{proof}

\msemve*
\begin{proof}
\begin{align}
	G_{MSE}(\theta) &= -  \frac{1}{2} \expectation_{(s, a) \sim \hat{D}^\pi_\mu} 
	\nabla_\theta \left[ ( \hat{T}(s, a) -  q_\theta(s, a) )^2 \right] \nonumber \\
	&= \expectation_{(s, a) \sim \hat{D}^\pi_\mu} 
	\left[ \left(\hat{T}(s, a) - q_\theta(s, a)\right) \nabla_\theta q_\theta(s, a) \right] \nonumber \\
	&= \expectation_{(s, a) \sim \hat{D}^\pi_\mu} \left[ \hat{G}_{MSE}(s, a, \theta) \right] \label{eqn:gmse_eq} 
\end{align}
This proves the statement about $\hat{G}_{MSE}$. For the  claim regarding $\hat{G}_{MVE}$:
\begin{align}
	G_{MVE}(\theta) &= - \nabla_\theta \frac{1}{2} \expectation_{s} \variance_{a}\Big[\hat{T}(s, a) - q_\theta(s, a) \Big] \\
			&= - \frac{1}{2} \expectation_{s} \nabla_\theta \variance_{a|s \sim \hat{\pi}}\Big[\hat{T}(s, a) - q_\theta(s, a) \Big] \\
			&= -\frac{1}{2} \expectation_{s} \left[ \nabla_\theta \expectation_{a|s \sim \hat{\pi}} \left[ (\hat{T}(s, a) - q_\theta(s, a))^2 \right] - \nabla_\theta \left( \expectation_{a|s \sim \hat{\pi}} \left[\hat{T}(s, a) - q_\theta(s, a)\right]\right)^2 \right] \\
			&= \expectation_{s} \left[ \expectation_{a|s \sim \pi} \Big[ \Big(\hat{T}(s, a) - q_\theta(s, a)\Big)\nabla_\theta q_\theta(s, a)\Big] - \expectation_{a|s \sim \pi} \Big[\hat{T}(s, a) - q_\theta(s, a) \Big] \expectation_{u|s \sim \pi} \Big[ \nabla_\theta q_\theta(s, u) \Big] \right] \\
			&= \expectation_{(s, a)}\left[ \Big(\hat{T}(s, a) - q_\theta(s, a)\Big)\Big(\nabla_\theta q_\theta(s, a) - \expectation_{u|s \sim \pi} \left[ \nabla_\theta q_\theta(s, u) \right]\Big) \right] \\
			&= \expectation_{(s, a)}\left[ \Big(\hat{T}(s, a) - q_\theta(s, a)\Big) \nabla_\theta \log \pi_\theta(a|s) \right] \\
			&= \expectation_{(s, a)} \left[ \hat{G}_{MVE}(s, a, \theta) \right]
\end{align}
\end{proof}

\ppo*

\begin{proof}
Assume that the batch of sample data is generated from a policy $\pi_b$ with parameters $b$.
Let $\hat{G}_{PPO}$ denote the PPO surrogate objective sample gradient. This gradient can be simplified as follows to recover the scaling function equivalence.
The PPO surrogate objective is defined for a given clipping parameter $\epsilon$ as:
$$J_{PPO}(s, a, \theta) = \left(\frac{\pi_\theta(a|s)}{\pi_b(a|s)}\hat{A}, \text{clip}\left(\frac{\pi_\theta(a|s)}{\pi_b(a|s)}, 1-\epsilon, 1 + \epsilon\right) \hat{A} \right)  $$
We will focus on computing the gradient, $\hat{G}_{PPO}(s, a, \theta) = \nabla_\theta J_{PPO}(s, a, \theta)$ below.
Note that $\Delta_O = \log \frac{\pi_\theta(a|s)}{\pi_b(a|s)}$, and $\Delta_R = \hat{A}$ under the specified assumptions. Therefore, 
\begin{align}
	& \hat{G}_{PPO}(s, a, \theta) = \nabla_\theta \left( \min \left(\frac{\pi_\theta(a|s)}{\pi_b(a|s)}\Delta_R, \text{clip}\left(\frac{\pi_\theta(a|s)}{\pi_b(a|s)}, 1-\epsilon, 1 + \epsilon\right) \Delta_R \right) \right) \\
	&= \nabla_\theta \left( \mathbbm{1}_{\Delta_R > 0} \min \left( \frac{\pi_\theta(a|s)}{\pi_b(a|s)}, 1 + \epsilon \right) \Delta_R + \mathbbm{1}_{\Delta_R < 0} \max \left( 1- \epsilon, \frac{\pi_\theta(a|s)}{\pi_b(a|s)} \right) \Delta_R \right) \\
	 &= \nabla_\theta \Big( \left( \mathbbm{1}_{\Delta_R > 0} \mathbbm{1}_{\Delta_O < \log(1+\epsilon)} + \mathbbm{1}_{\Delta_R < 0} \mathbbm{1}_{\Delta_O > \log(1 + \epsilon)}\right) \frac{\pi_\theta(a|s)}{\pi_b(a|s)} \Delta_R ~~ + \nonumber \\
	 & \hspace{1in} (1+ \epsilon) \mathbbm{1}_{\Delta_R > 0} \mathbbm{1}_{\Delta_O \ge \log(1+\epsilon)} \Delta_R + (1 - \epsilon) \mathbbm{1}_{\Delta_R < 0} \mathbbm{1}_{\Delta_O \le \log(1 - \epsilon)} \Delta_R \Big) \\
	 &= \left( \mathbbm{1}_{\Delta_R > 0} \mathbbm{1}_{\Delta_O < \log(1+\epsilon)} + \mathbbm{1}_{\Delta_R < 0} \mathbbm{1}_{\Delta_O > \log(1 + \epsilon)}\right) \nabla_\theta \frac{\pi_\theta(a|s)}{\pi_b(a|s)} \Delta_R + \textbf{0}  ~~~(\because \text{ no gradient from the second term.}) \\
	 &= \tau_\epsilon(\Delta_O, \Delta_R) \frac{\pi_\theta(a|s)}{\pi_b(a|s)} \nabla_\theta \log \pi_\theta(a|s)  \Delta_R \\
	 &= \tau_\epsilon(\Delta_O, \Delta_R) e^{\Delta_O} \Delta_R \nabla_\theta \log \pi_\theta(a|s)
\end{align}
\end{proof}

\subsection{Design for the parametric scaling function approximations}
Consider the following lower bound to $e^{\Delta_O}(e^{\Delta_R} - 1)$ which is also exact upto second order around $(\Delta_O, \Delta_R) = (0, 0)$. Let $X \sim U[\Delta_O, \Delta_O + \Delta_R]$ be a continuous RV with uniform density. Then:
$$\expectation[e^X] = \int_{\Delta_O}^{\Delta_O + \Delta_R}
\frac{e^x}{\Delta_R} dx = \frac{e^{\Delta_O + \Delta_R} -
e^{\Delta_O}}{\Delta_R} \mbox{ and } e^{\expectation[X]} = e^{\Delta_O + \frac{\Delta_R}{2}}$$
Jensen's inequality implies that
$$0 < e^{\Delta_O + \frac{\Delta_R}{2}} \le \frac{e^{\Delta_O + \Delta_R} - e^{\Delta_O}}{\Delta_R} ~~ \forall ~~ \Delta_O, \Delta_R \in \mathbb{R} $$
Consider the following approximation inspired by the above inequality:
\begin{equation} \label{eqn:japprox}
e^{\Delta_O + \frac{\Delta_R}{2}} \approx \frac{e^{\Delta_O + \Delta_R} - e^{\Delta_O}}{\Delta_R} 
\end{equation}
We then consider the following sequence of approximations:
\begin{align}
	e^{\Delta_O}(e^{\Delta_R} - 1) &= e^{\Delta_O + \Delta_R} - e^{\Delta_O} \\
				       &\approx \Delta_R e^{\Delta_O + \frac{\Delta_R}{2}} ~~ \mbox{ Using Eq } \ref{eqn:japprox}\\
				       &\ge \Delta_R \max \left(1 + \Delta_O + \frac{\Delta_R}{2}, 0\right) ~~ \mbox{ since $e^x \ge \max(1 + x, 0) ~~ \forall x \in \mathbb{R}$ } \label{eqn:approxd1}
\end{align}
Unfortunately, $f_0(\Delta_O, \Delta_R) \triangleq \Delta_R \max \left(1 + \Delta_O + \frac{\Delta_R}{2}, 0\right)$ does not satisfy the monotonicity constraint in $\Delta_R$ listed in \ref{assumption:constraints} in all cases.
To check this, we can verify that it is non-decreasing (and also convex) in $\Delta_R$ if and only if $1 + \Delta_O \le 0$.
Equation (\ref{eqn:ml1}) makes a fix to $f_0$ to ensure that the monotonicity constraints from Assumption \ref{assumption:constraints} hold while also remaining a good second order approximation when $\Delta_O, \Delta_R$ are small. 

\textbf{Fix for non-monotonicity($f_{MLA}$):} First notice that the non-monotonicity only happens when $1 + \Delta_O > 0$. 
In that case, the minimum of the quadratic in $y$, $q(x, y) = y (1 + x + y/2)$ occurs for $y_{min}(x) \triangleq - (1 + x) < 0$. 
However, in this case $1 + x + y_{min}(x)/ 2 = (1 + x)/2 > 0$.
So $f(x, y) = q(x, y)$ in a left neighborhood of $y_{min}(x) = -(1 + x)$, which means it must be decreasing in $y$ in that range. 
To avoid this, we make a fix $f(x, y) \triangleq q(x, y_{min}(x)) = -(1 + x)^2/2$, when $y < y_{min}(x) = - (1 + x)$ and $1 + x > 0$.  The resulting $f(x, y)$ is the definition of $f_{MLA}(x, y)$ in Equation (\ref{eqn:ml1}).

\textbf{Derivation of $f_{MLA(\alpha_0, \alpha_r)}$:}
In the policy parameterization, we did not include any temperature parameter for simplicity. However, by considering a generalization of these as free parameters, 
we may consider scaling functions where the inputs are replaced with $\alpha_o \Delta_O$ and $\alpha_r \Delta_R$ respectively. 
The same sequence of approximations resulting in Equation (\ref{eqn:approxd1}) now result in the constraint violating function\footnote{$(x)_+$ denotes $\max(x, 0)$} $f_1(x, y) = \alpha_r y(1 + \alpha_o x + \alpha_r y / 2)_+$. 
Without loss of generality, this is equivalent to $f_1(x, y) = y(1 + \alpha_o x + \alpha_r y)_+$, 
because (1) the constant $\alpha_r$ scaling the outer $y$ has no impact on the overall update, and (2) we can redefine the remaining $\alpha_r$ inside to include the 1/2 for simplicity.
We are now left with fixing the non-monotonicity of $f_1(x, y) = y\max(1 + \alpha_0 x + \alpha_r y, 0)$ for certain ranges of $y$ when $1 + \alpha_0 x < 0$. 
However, we do this slightly differently compared to how it was done for $f_{MLA}$ above, in order to preserve the linear component  in $y$, 
as this allows for a more intuitive understanding as being the component that matches the identity scaling function. \\

\textbf{Fix for non-monotonicity:}  As before, consider the quadratic in $y$, $q(x, y) \triangleq y(1 + \alpha_o x + \alpha_r y)$ 
and its argmin 
$y_{min}(x) = \frac{-(1 + \alpha_o x)}{2\alpha_r}$ in the problematic case when $1 + \alpha_o x > 0$. Observe that when $y=y_{min}$, we have $1 + \alpha_o x + \alpha_r y = \frac{1+ \alpha_o x}{2}$. Recall that the problem
occurs due to the quadratic dependence on $y$ from the second term for $y < y_{min}$. 
Instead of completely eliminating the dependence on $y$, we can continue a monotonic linear relationship by simply clipping the second term below by the value it takes at $y_{min}$ which is
$1 + \alpha_o x + \alpha_r y = \frac{1+ \alpha_o x}{2}$ . 
However, recall that we should only perform this clipping of the second term below by $\frac{1 + \alpha_o x}{2}$
in the case when $1 + \alpha_o x > 0$. 
This can be done simply by clipping the second term by $\frac{(1 + \alpha_o x)_+}{2}$ as this reduces to the original floor of $0$ in the case where no fix was needed.
This results in the proposed $f_{MLA(\alpha_o, \alpha_r)}(x, y) = y \max \left(1 + \alpha_o x + \alpha_r y, \frac{(1 + \alpha_o x)_+}{2}\right)$.

\begin{figure}
	\centering
	\begin{subfigure}{.45\textwidth}
		\centering
		\includegraphics[width=.9\linewidth]{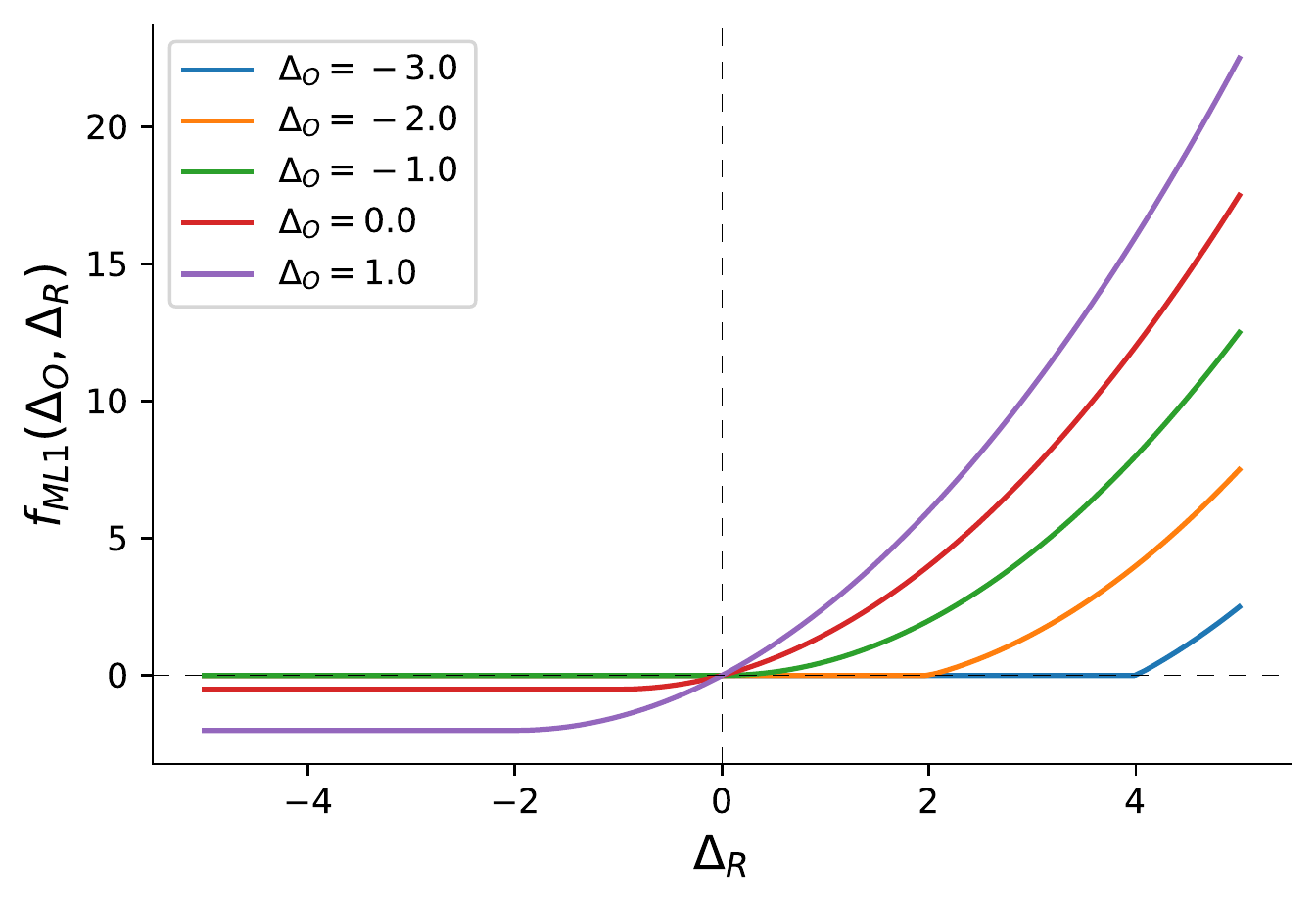}
		\caption{Projections on $\Delta_R$}
		\label{fig:slice_r_1}
	\end{subfigure}%
	\begin{subfigure}{.45\textwidth}
		\centering
		\includegraphics[width=.9\linewidth]{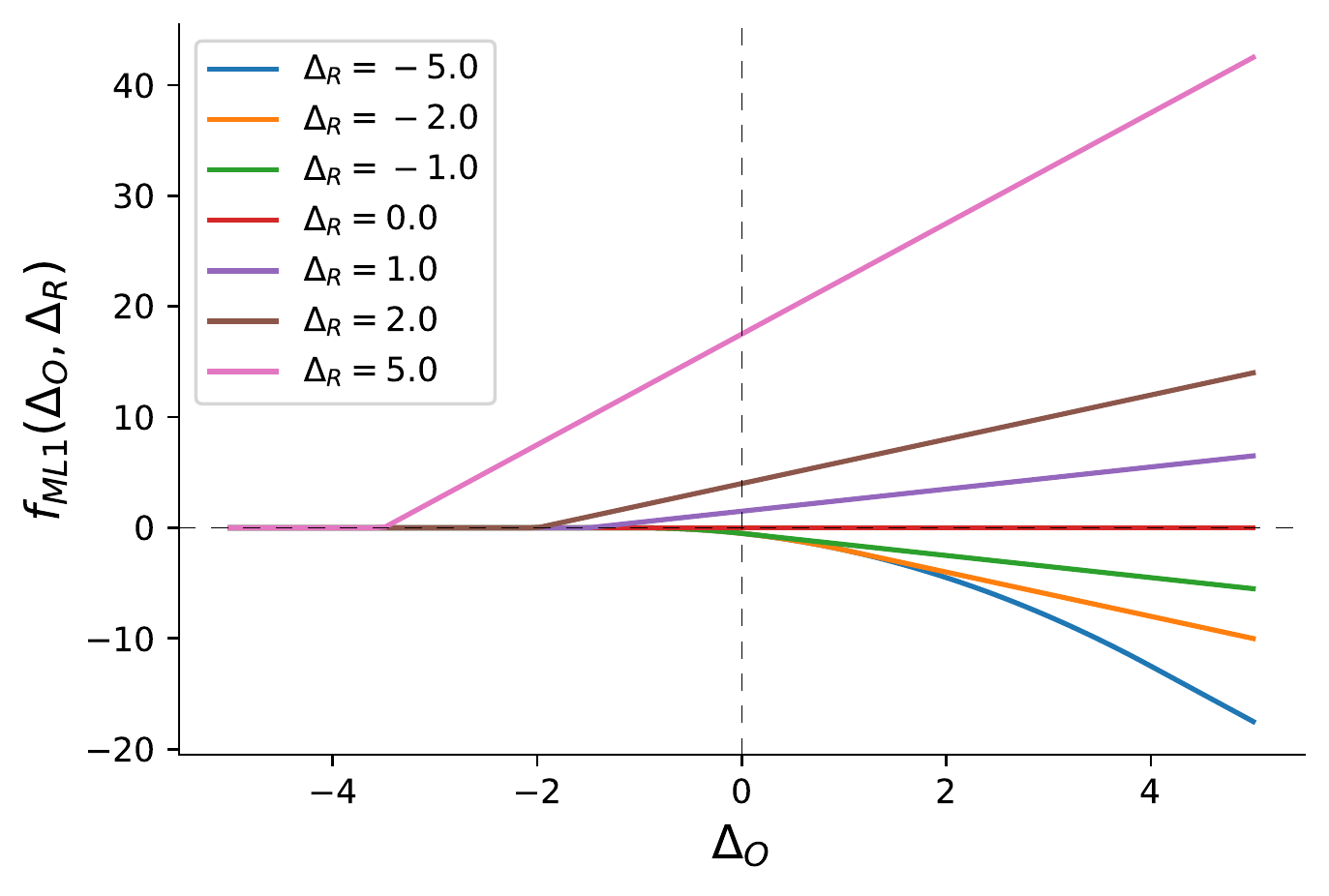}
		\caption{Projections on $\Delta_O$}
		\label{fig:slice_o_1}
	\end{subfigure}
	\caption{Visualization for $f_{MLA}(\Delta_O, \Delta_R)$ (Equation \ref{eqn:ml1}) }
	\label{fig:ml1}
\end{figure}
\begin{figure}
	\centering
	\begin{subfigure}{.45\textwidth}
		\centering
		\includegraphics[width=.9\linewidth]{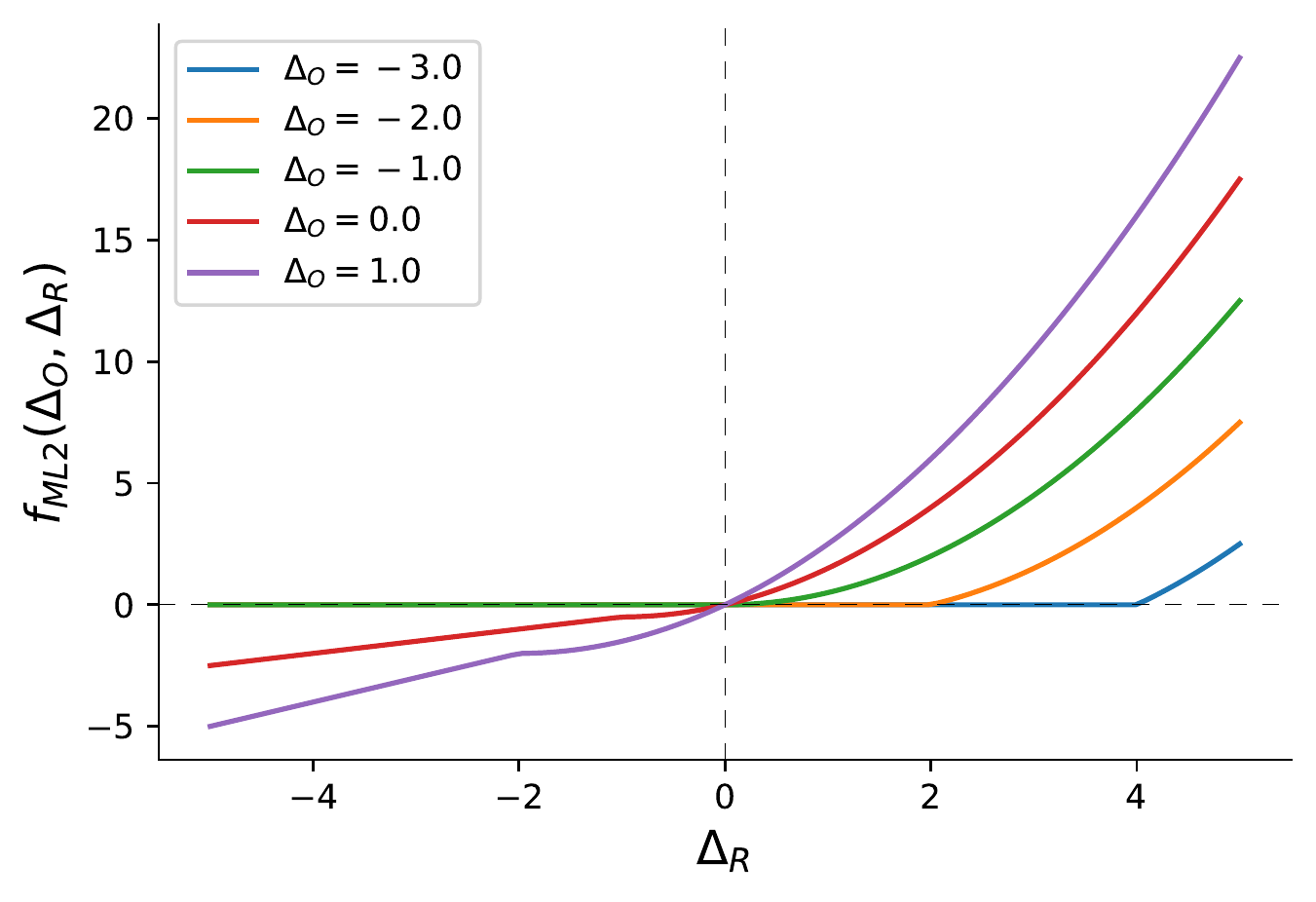}
		\caption{Projections on $\Delta_R$}
		\label{fig:slice_r_2}
	\end{subfigure}%
	\begin{subfigure}{.45\textwidth}
		\centering
		\includegraphics[width=.9\linewidth]{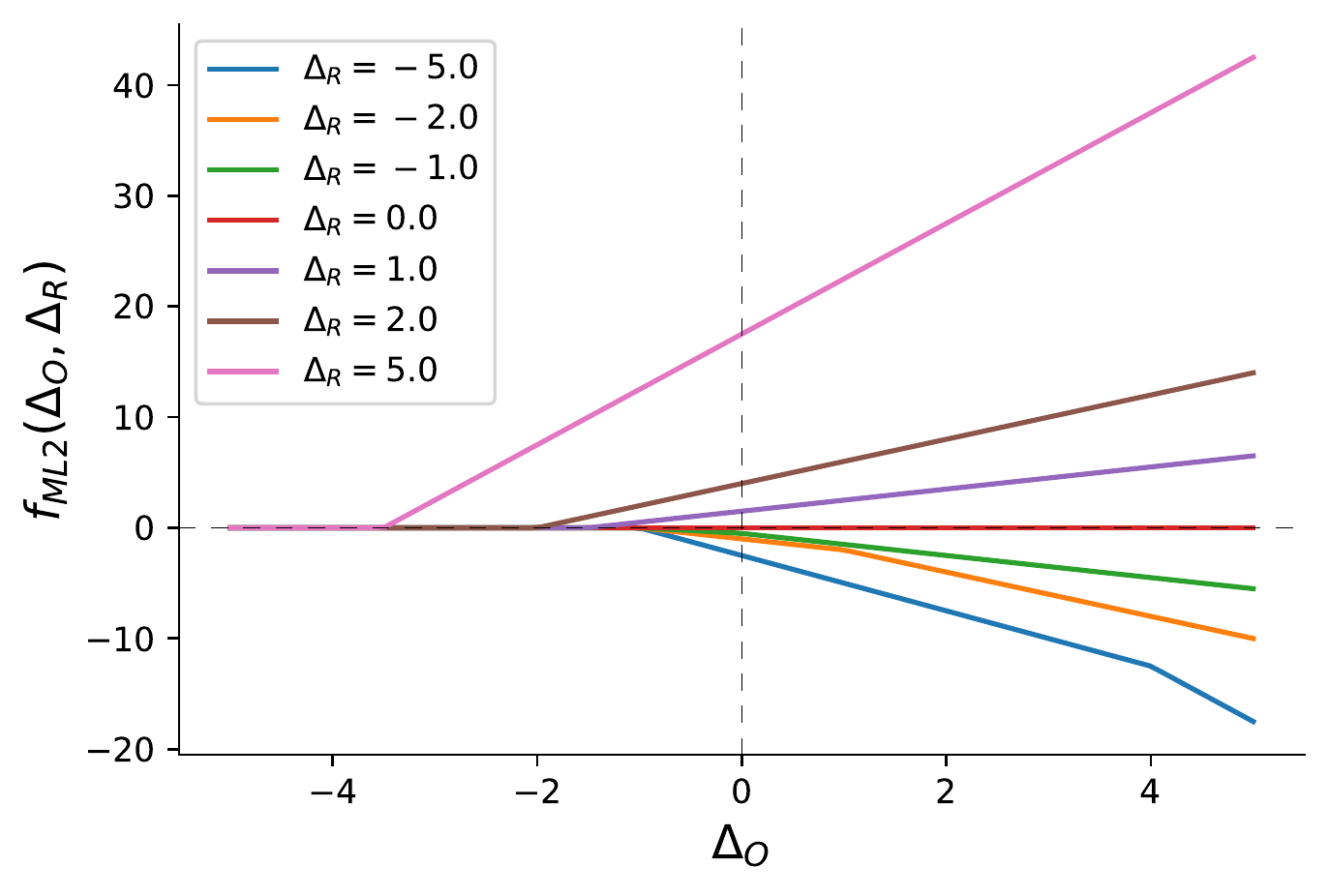}
		\caption{Projections on $\Delta_O$}
		\label{fig:slice_o_2}
	\end{subfigure}
	\caption{Visualization for $f_{MLA(1, 0.5)}(\Delta_O, \Delta_R)$ (Equation \ref{eqn:ml2} for $\alpha_o = 1, \alpha_r=0.5$.)}
	\label{fig:ml2}
\end{figure}

\begin{figure}
	\centering
	\begin{subfigure}{0.45\textwidth}
		\centering
		\includegraphics[width=.9\linewidth]{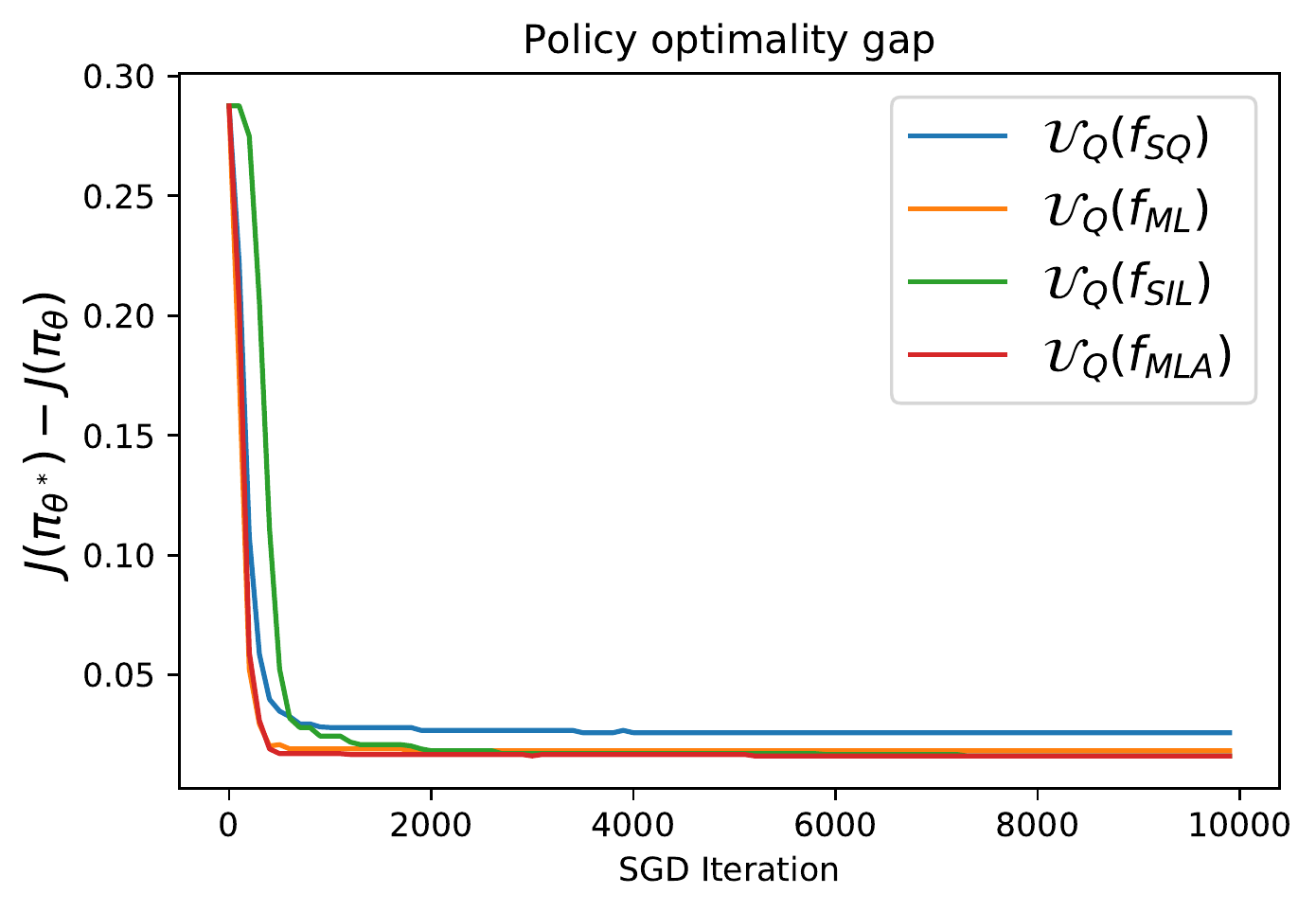}
		\caption{Learning curves for $\mathcal{U}_Q(f)$: Col 1, Table \ref{tab:table1}}
		\label{fig:gap_all}
	\end{subfigure}%
	\begin{subfigure}{.45\textwidth}
		\centering
		\includegraphics[width=.9\linewidth]{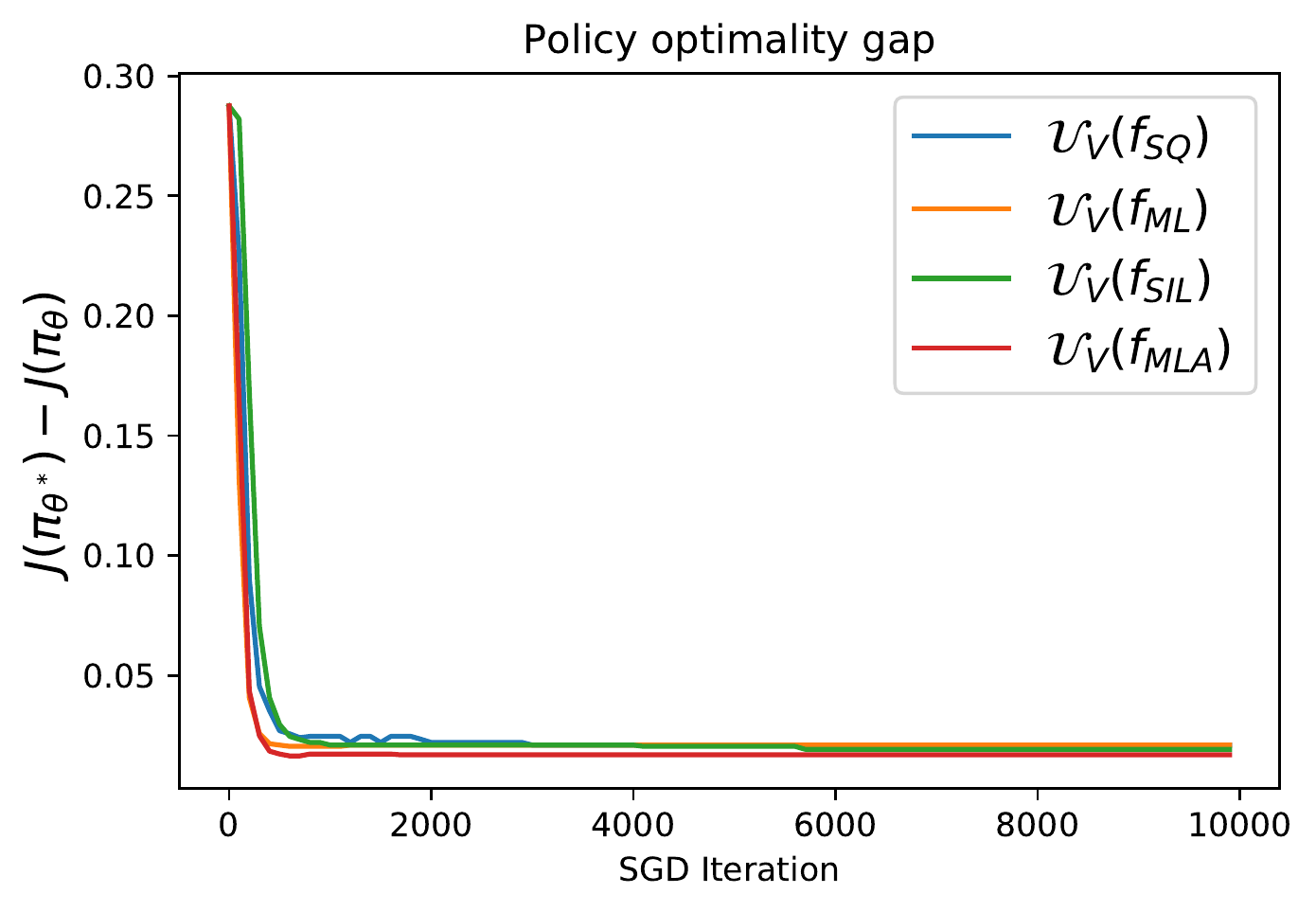}
		\caption{Learning curves for $\mathcal{U}_V(f)$: Col 2, Table \ref{tab:table1}}
		\label{fig:gap_sq}
	\end{subfigure}
	\begin{subfigure}{.45\textwidth}
		\centering
		\includegraphics[width=.9\linewidth]{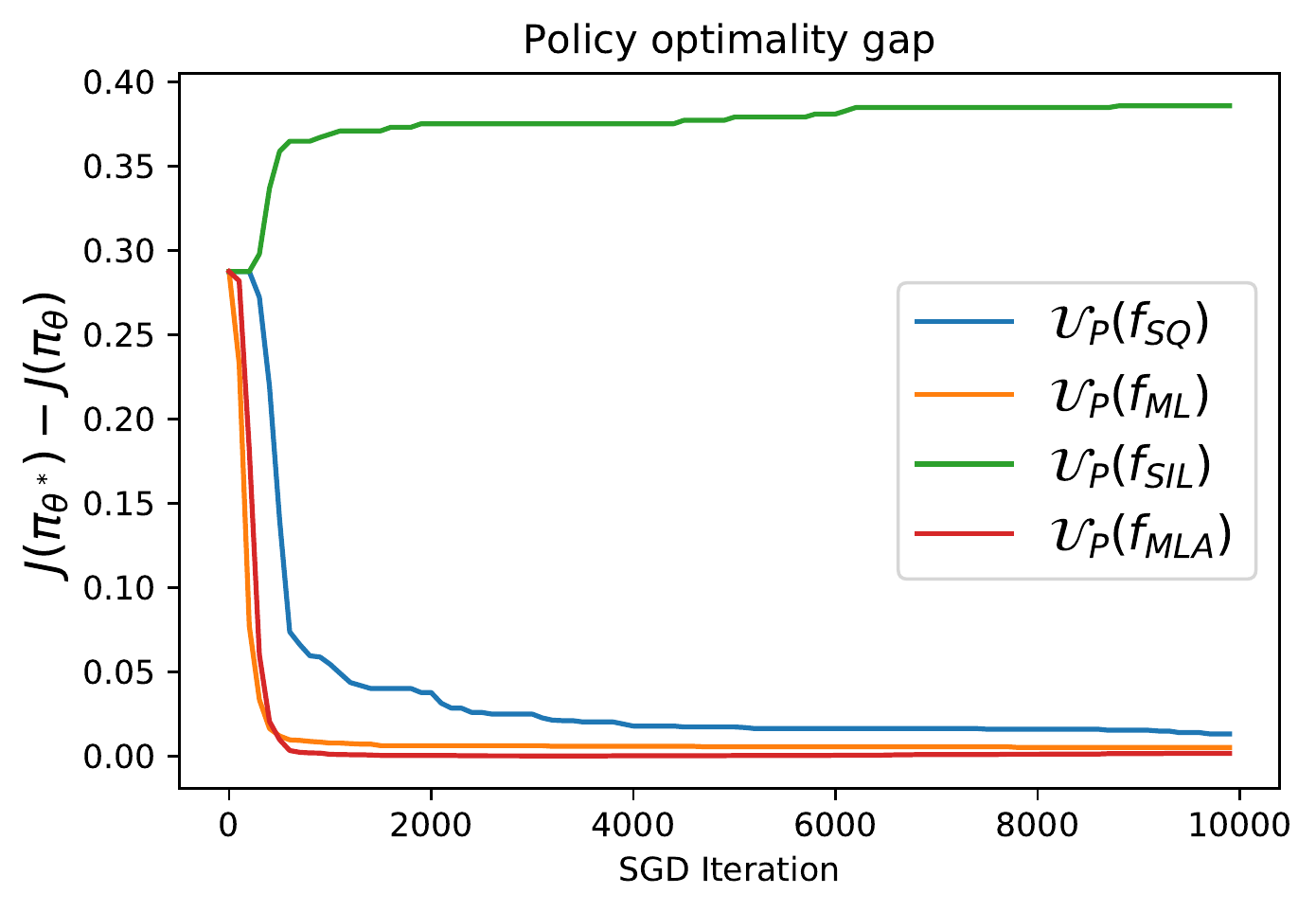}
		\caption{Learning curves for $\mathcal{U}_{P}(f)$: Col 3, Table \ref{tab:table1}}
		\label{fig:gap_pgpb}
	\end{subfigure}
	\caption{Gap to the optimal policy reward for the twelve updates listed in Table \ref{tab:table1} grouped by the update type in each column, and compared across scaling functions $f \in \{f_{SQ}, f_{ML}, f_{SIL}, f_{MLA}\}$}
	\label{fig:gap}
\end{figure}

\begin{figure}
	\centering
	\begin{minipage}{.6\textwidth}
		\centering
		\includegraphics[width=0.9\linewidth]{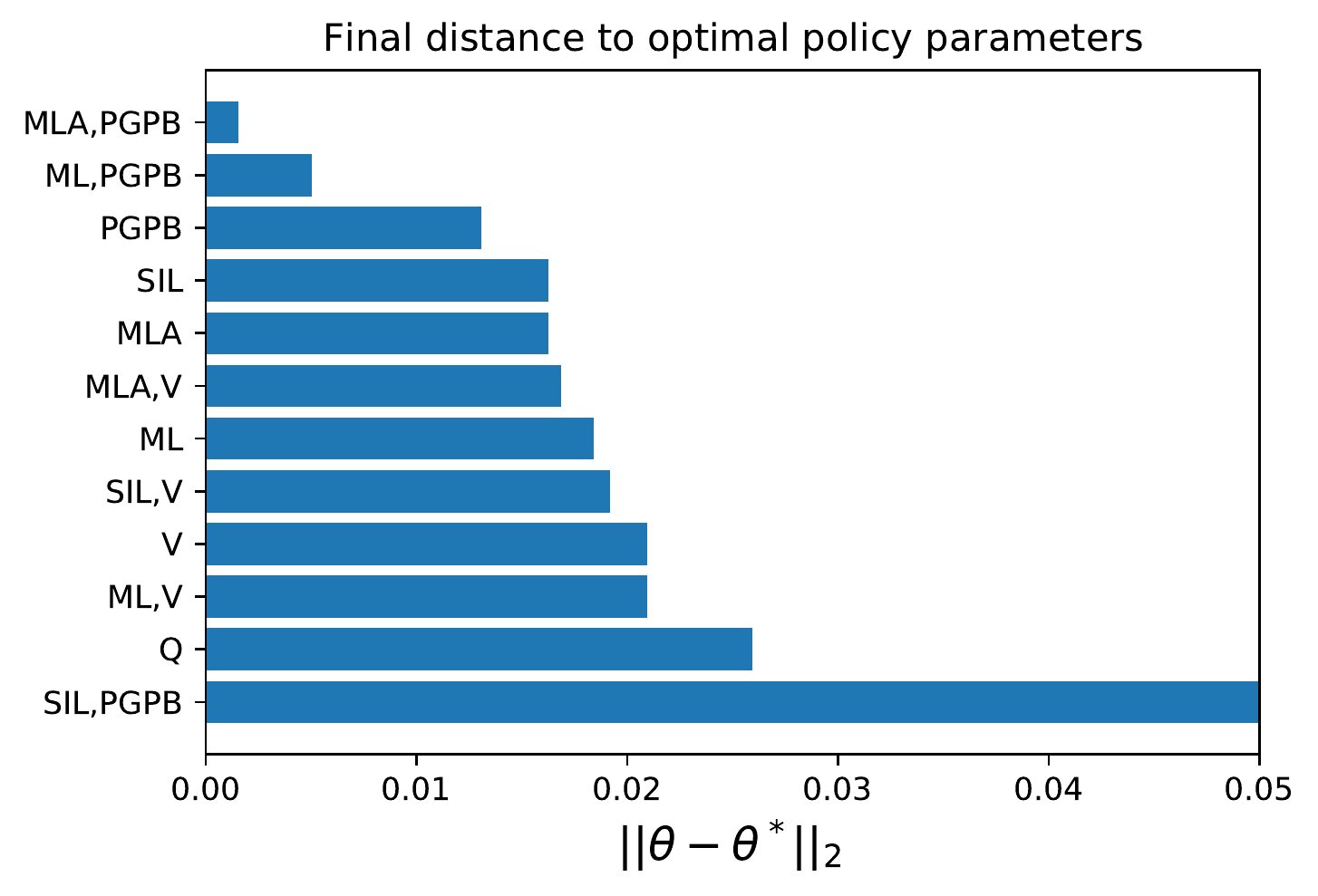}
		\caption{Final Euclidean distance to the optimal parameters, $\theta^* = (1, 1)$ for the tweleve update rules considered.}
		\label{fig:err}
	\end{minipage}%
\end{figure}

\section{Table of update rule combinations evaluated for the 2D synthetic experiment}
The full range of gradients characterized are summarized in Table~\ref{tab:table1}, which considers the form-axis along the columns for the three types of updates (squared error minimization, variance minimzation and PGPB), in conjunction with several scaling functions along the row axis.  
\begin{table}[!ht]
\begin{center}
\caption{A summary of the various updates evaluated in the synthetic 2D experiment.}
	\renewcommand{\arraystretch}{1.5}
\label{tab:table1}
	\begin{tabular}{|c|c|c|c|}
		\hline
		\multicolumn{4}{|l|}{$\mathcal{U}_{Q}(f) \equiv f(\Delta_O,\Delta_R)\Big( \nabla_\theta q_\theta(s, a) \Big)$} \\
		\multicolumn{4}{|l|}{$\mathcal{U}_{V}(f) \equiv f(\Delta_O,\Delta_R) \Big(\nabla_\theta q_\theta(s, a) - \expectation_{u|s \sim \pi} \left[ \nabla_\theta q_\theta(s, u) \right]\Big)$ }\\
		\multicolumn{4}{|l|}{$\mathcal{U}_{P}(f) \equiv f(\Delta_O,\Delta_R) \Big(\nabla_\theta q_\theta(s, a) - \expectation_{u|s \sim \pi} \left[ \nabla_\theta q_\theta(s, u) \right]\Big) + \nabla_\theta \expectation_{u|s \sim \pi_\theta} \left[ \hat{q}_\theta(s, u) \right]$} \\
		\hline
		\hline
		$f$ & $\mathcal{U}_Q(f)$ & $\mathcal{U}_V(f)$ & $\mathcal{U}_P(f)$\\
		\hline
		$f_{SQ}(x, y) \triangleq e^x y$ & $\hat{G}_{Q}$& $\hat{G}_{V}$& $\hat{G}_{PGPB}$ \\
		\hline
		$f_{ML}(x, y) \triangleq e^x (e^y - 1) $ & $\hat{G}_{ML}$ & $\hat{G}_{ML, V}$ & $\hat{G}_{ML, PGPB}$ \\
		\hline
		$f_{SIL}(x, y) \triangleq e^{x} \max \left(y, 0\right)$ & $\hat{G}_{SIL}$ & $\hat{G}_{SIL, V}$ & $\hat{G}_{SIL, PGPB}$ \\
		\hline
		$f_{MLA}(x, y) \triangleq
		 \begin{cases}
			-\frac{1}{2}(1 + x)^2 , & \text{if }1 + x + y \le 0 < 1 + x \\
		y \max \left(1 + x + \frac{y}{2}, 0\right) ,              & \text{otherwise}
	\end{cases}$ & $\hat{G}_{MLA}$ & $\hat{G}_{MLA, V}$ & $\hat{G}_{MLA, PGPB}$\\
		\hline
 \end{tabular}
\end{center}
\end{table}

\begin{figure}
	\centering
		\includegraphics[width=.5\linewidth]{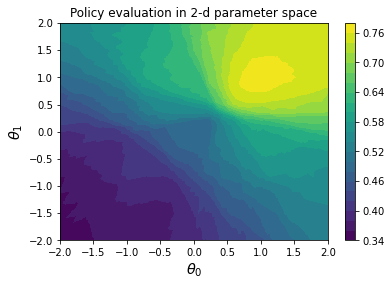}
		\caption{Objective landscape}
		\label{fig:objective}
\end{figure}

\section{Experiment Details}

\subsection{Tabular Experiments}

\begin{figure}
	\centering
	\includegraphics[width=0.4\linewidth]{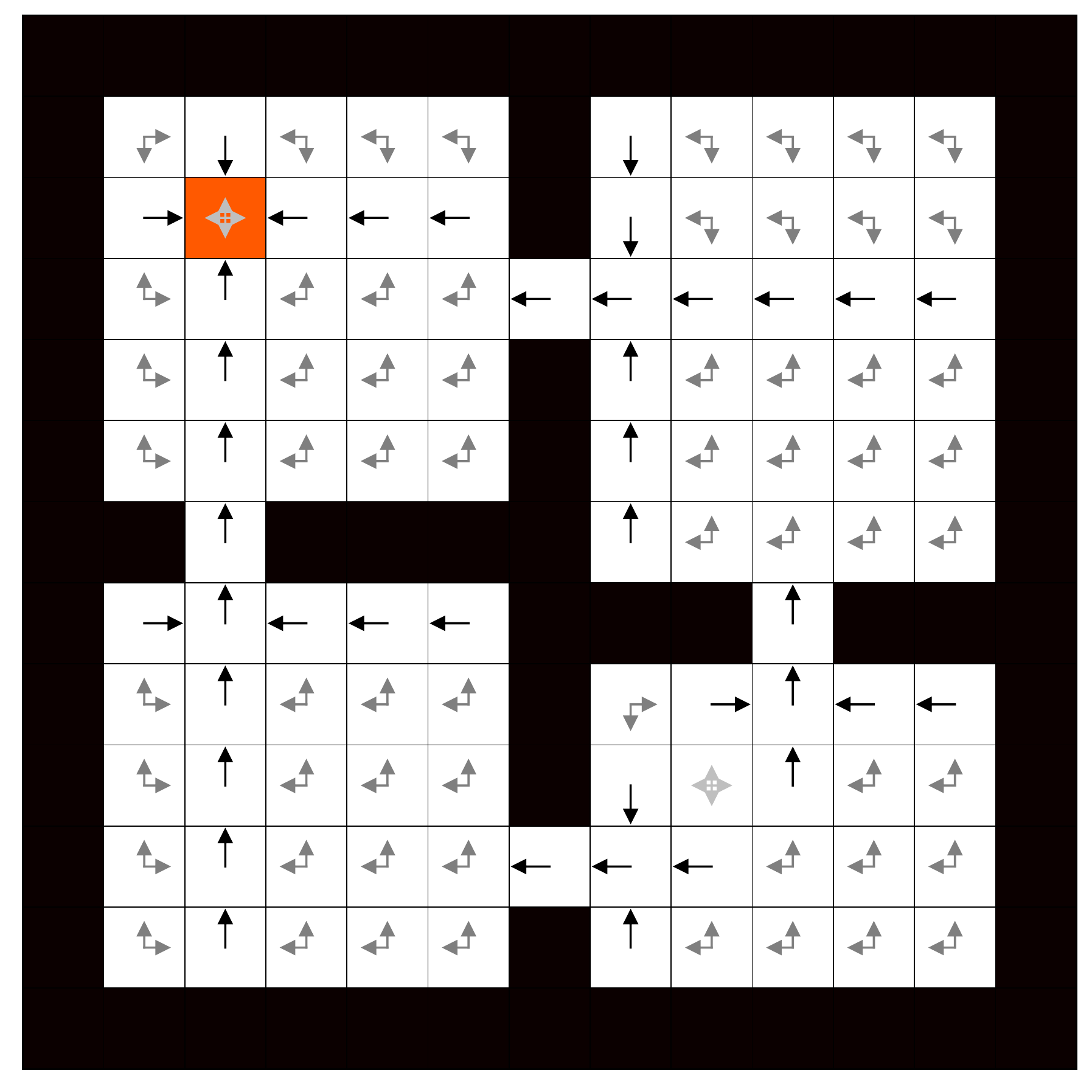}
	\caption{The FourRoom Environment}
	\label{fig:fourroom-env}
\end{figure}

The FourRoom environment is shown in Figure~\ref{fig:fourroom-env}.
Both policy gradient (PG) and $Q$-learning (QL) use SGD as the optimizer. 
The learning rate for PG is 0.1 for both the actor and the critic, while the learning rate for QL is 0.01. 
Figure~\ref{fig:fourroom-results} reports the exponential running average return with exponent discount of 0.01.

\subsection{Continuous Control Experiments}\label{appendix:cce}
The full set of update equations used are summarized below:
\begin{align*}\label{eqn:mla_ppo}
	\hat{G}_{MLA-PPO}(s, a, \theta) &= f_{MLA-PPO}(\Delta_O, \Delta_R) \nabla_\theta \log \pi_\theta(a|s) + \beta \nabla_\theta H(\pi_\theta(.|s)), \text{ where:} \\
	f_{MLA-PPO}(x, y) &= f_{MLA(\alpha_o, \alpha_r)}(x, y) \tau_\epsilon(x, y), \text{ where:}  \\
	f_{MLA(\alpha_o, \alpha_r)}(x, y) &=  y \max \left( 1 + \alpha_o x + \alpha_r y , \frac{(1 + \alpha_o x)_+}{2}\right) \nonumber  \\
	\tau_\epsilon(x, y) &= \mathbbm{1}_{y > 0} \mathbbm{1}_{x < \log(1 + \epsilon)} + \mathbbm{1}_{y <0} \mathbbm{1}_{x > \log(1 - \epsilon)} \\
	\Delta_R &= \hat{A} - \alpha \big(\log \pi(a|s) + H(\pi(.|s))\big)
\end{align*}

\begin{table}[!ht]
\begin{center}
\caption{Optimal hyper-parameter configurations for PPO - MLA on the MuJoCo Tasks. Note that the performance reported for the baseline PPO implementation uses its own independently tuned clipping $\epsilon$.}
	\renewcommand{\arraystretch}{1.5}
\label{tab:table2}
	\begin{tabular}{|c|c|c|c|c|}
		\hline
		Environment & $\alpha_r$ & $\alpha_o$ & $\alpha$ & $\epsilon$\\
		\hline
		Walker2d-v2 & $0.05$ & $0.1$ & $0.0$ & $0.2$ \\
		\hline
		Hopper-v2 & $1.0$ & $0.1$ & $0.01$ & $0.2$ \\
		\hline
		HalfCheetah-v2 & $0.5$ & $0.1$ & $0.1$ & $0.2$ \\
		\hline
		Humanoid-v2 & $1.0$ & $1.0$ & $0.0$ & $0.3$ \\
		\hline
 \end{tabular}
\end{center}
\end{table}